\documentclass[12pt]{article}

\usepackage{ulem}
\usepackage{mysty}
\usepackage{multirow}
\usepackage{ctable}
\newcommand{\gap}{\mathrm{Gap}}
\newcommand{\cert}{Z}
\newcommand{\bcert}{\bar Z_t}
\newcommand{\zmstar}{Z^\star}
\usepackage[noindentafter]{titlesec}
\usepackage{bold-extra}

\usepackage{float}

\definecolor{cornellred}{rgb}{0.7, 0.11, 0.11}
\definecolor{bostonuniversityred}{rgb}{0.8, 0.0, 0.0}
\definecolor{amaranth}{rgb}{0.9, 0.17, 0.31}

\newcommand{\rev}[1]{{#1}}

\newcommand{\finalrev}[1]{#1}

\usepackage{natbib}
\usepackage{url} 

\newcommand{\blind}{0}

\usepackage{calc}
\newsavebox\CBox
\newcommand\hcancel[2][0.5pt]{%
 \ifmmode\sbox\CBox{$#2$}\else\sbox\CBox{#2}\fi%
 \makebox[0pt][l]{\usebox\CBox}%
 \rule[0.5\ht\CBox-#1/2]{\wd\CBox}{#1}}

 \usepackage{amssymb,amsmath,amsthm,enumitem}

\begin{document}

\def\spacingset#1{\renewcommand{\baselinestretch}%
{#1}\small\normalsize} \spacingset{1}

\if0\blind
{
 \title{\bf Variable Screening for Sparse Online Regression}
 \author{Jingwei Liang\hspace{.2cm}\\
 Institute of Natural Sciences, Shanghai Jiao Tong University\\
 and \\
 Clarice Poon \\
 Department of Mathematical Sciences, University of Bath}
 \maketitle
} \fi

\if1\blind
{
 \bigskip
 \bigskip
 \bigskip
 \begin{center}
 {\LARGE\bf Title}
\end{center}
 \medskip
} \fi

\bigskip
\begin{abstract}
 Sparsity-promoting regularizers are widely used to impose low-complexity structure (e.g. $\ell_1$-norm for sparsity) to the regression coefficients of supervised learning. In the realm of deterministic optimization, the sequence generated by iterative algorithms (such as proximal gradient descent) exhibit ``finite activity identification" property, that is, they can identify the low-complexity structure of the solution in a finite number of iterations. However, many online algorithms (such as proximal stochastic gradient descent) do not have this property owing to the vanishing step-size and non-vanishing variance. In this paper, by combining with a screening rule, we show how to eliminate useless features of the iterates generated by online algorithms, and thereby enforce finite sparsity identification. One advantage of our scheme is that when combined with any convergent online algorithm, sparsity properties imposed by the regularizer can be exploited to improve computational efficiency. Numerically, significant acceleration can be obtained.
 \end{abstract}

\noindent%
\vfill

\newpage
\spacingset{1.5} 

\section{Introduction}

\subsection{Background}

Regression plays a fundamental role in various fields including machine learning, statistics and data science. Meanwhile, sparse regularizations, such as $\ell_1$-norm regularization, have been increasingly popular in recent years. 
In this paper, we are interested in the following sparsity-promoting regression problem
\begin{equation}\label{eq:onlineprob}\tag{$P_{\lambda}$}
\min_{\beta \in\RR^n} 
~ \bBa{ P_{\lambda}(\beta) = \lambda \Omega(\beta) + F(\beta) } , \qwhereq F(\beta) \eqdef \EE_{(x,y)} [ f( x^\top \beta; y) ].
\end{equation}
The expectation is taken over random variable $(x,y)$ whose probability distribution $\Lambda$ is supported on some compact domain $\Xx\times \Yy\subset \RR^n\times \RR$ and $\lambda>0$ is a trade-off parameter to balance the loss $F$ and the sparsity promoting regularizer $\Omega$.

Popular choices of the loss function $f$ include the squared loss, logistic loss and the squared hinge loss, while popular choices for $\Omega$ include the $\ell_1$-norm for enforcing sparsity~\cite{tibshirani1996regression}, the $\ell_{1,2}$-norm for enforcing group sparsity~\cite{yuan2006model} and the $\ell_1+\ell_{1,2}$-norms for enforcing sparsity within groups~\cite{simon2013sparse}.
Throughout this paper, we assume the following basic assumptions:
\begin{enumerate}[leftmargin=4em,label= ({\textbf{H.\arabic{*}})},ref= \textbf{H.\arabic{*}}]
\item \label{A:F} 
For each $y$, $f_y \eqdef f(\cdot; y) : \RR\to\RR$ is convex, differentiable, and has $L$-Lipschitz continuous gradient for some $L>0$.
\item \label{A:Omega} 
The regularization function $\Omega: \RR^n\to [0,+\infty)$ is a convex and group decomposable norm \rev{(with non-overlapping groups)}. That is, given $\beta\in\RR^n$ and a partition $\Gg$ on $\ens{1,\ldots, n}$ such that $\beta = (\beta_g)_{g\in\Gg}$, we have 
\[\Omega(\beta) \eqdef \msum_{g\in\Gg}\Omega_g(\beta_g)\] 
where $\Omega_g$ is a norm on $\RR^{n_g}$ with $n_g$ being the cardinality of $\beta_g$.
\end{enumerate}

\paragraph{Empirical loss minimization}
In practice, instead of minimizing the loss function $F$ over the distribution $\Lambda$, one can draw samples from $\Lambda$ and deals with the {\em empirical loss} 
\[
F_\eta(\beta) \eqdef \msum_{i=1}^m \eta_i f(x_i^\top \beta,y_i)
\]
where $m$ samples $\ens{x_i,y_i}_{i=1}^m \in (\RR^n)^m\times \RR^m$ are drawn from $\Lambda$, with positive weights $\eta_i$ which sum to 1. A popular choice of $(\eta_i)_{i}$ is uniform weights, \ie $\eta_i \equiv \frac{1}{m} $. 
Correspondingly, \eqref{eq:onlineprob} becomes the following regularized empirical loss
\begin{equation}\label{eq:finitesum}\tag{$P_{\lambda,\eta}$}
\min_{\beta \in \bbR^n}~ \bBa{ P_{\lambda,\eta}(\beta) = \lambda \Omega(\beta) + \msum_{i=1}^m \eta_i f(x_i^\top \beta,y_i) } . 
\end{equation}

\subsubsection{Dimension reduction via (safe) screening}
The purpose of using sparsity-promoting regularizers is so that the solution of the optimization problem \eqref{eq:finitesum} has as few non-zeros coefficients as possible. 
In high dimensional statistics, (safe) screening techniques are popular approaches for filtering out features whose corresponding coefficients are $0$, hence achieving dimension reduction; See~\cite{ghaoui2010safe,tibshirani2012strong,ndiaye2017gap} and the references therein. 
Safe feature elimination was first proposed by 
\cite{ghaoui2010safe} for $\ell_1$-norm regularized problems. The rules introduced were static, where features are screened out as a preprocessing step, and sequential rules where one solves a sequence of optimization problems with a decreasing list of parameters $\ens{\lambda_j}_{j}$, so that solutions of an optimization problem with $\lambda_{k+1}$ are used to screen out features when solving with $\lambda_k$. Since this work, several extensions have been proposed~\cite{liu2014safe, wang2014safe,xiang2011learning}.

Dynamic screening rules were later proposed by~\cite{bonnefoy2015dynamic}, where the safe screening regions are updated along the iterates of a solver. Another work in this direction is the so-called gap-safe rules~\cite{ndiaye2015gap, ndiaye2017gap} where the calculation of the safe regions along the iterates are done via primal-dual duality gap. The presented work is largely inspired by~\cite{ndiaye2017gap}, where we dynamically construct safe regions by computing an ``online'' primal-dual duality gap.

\subsection{Our contributions}
Though screening techniques are algorithm agnostic, they have been investigated mostly for the deterministic or batched algorithms (where one evaluates the full gradient $\nabla F$ or $\nabla F_\eta$ at each iteration).
For large-scale problems, batched methods (\eg proximal gradient descent) may be impractical and it is preferable to use online ones~\cite{bottou2004large}: 
that at each iteration step $t$, draw a sample $(x_t,y_t)$ randomly from the distribution $\Lambda$ and perform the update
 \begin{equation}\label{eq:sgd}
 \beta_{t+1} = \beta_t - \gamma_t\Pa{ f_{y_t}'(x_t^\top \beta_t) x_t + \lambda Z_t}
 \end{equation}
 where $Z_t \in \partial \Omega(\beta_t)$ is a subgradient (see Eq. \eqref{eq:sub-diff}). This is a special instance of stochastic gradient descent and dates back to
\cite{robbins1951stochastic,kiefer1952stochastic}.

To deal with the non-smoothness imposed by the regularizer $\Omega$, various stochastic schemes are proposed in the literature, such as \textit{truncated gradient}~\cite{langford2009sparse} or stochastic versions of proximal gradient descent~\cite{duchi2009efficient} (Prox-SGD) 
\[
\beta_{t+1} = \prox_{\lambda\gamma_t \Omega} \Pa{ \beta_t - \gamma_t f_{y_t}'(x_t^\top \beta_t) x_t } ,
\]
where $\prox_{\gamma\Omega}(\cdot) = \argmin_{\beta} \gamma \Omega(\beta)+ \frac12 \norm{\beta - \cdot}^2$ is called the {\it proximal mapping} of $\Omega$, 
and has closed form expressions for the above sparsity promoting regularizers~\cite{combettes2011proximal}. 
Note that Prox-SGD is equivalent to \eqref{eq:sgd} by taking $Z_t \in \partial \Omega(\beta_{t+1})$.
However, for the standard Prox-SGD, due to the vanishing step-size $\gamma_t$ and non-vanishing variance in the stochastic gradient estimates, the generated sequence $\seq{\beta_t}$ tends to have full support for all $t$ even though the sought after solution is sparse~\cite{xiao2010dual,lee2012manifold}; see also~\cite{poon2018local} for an explicit example. As a result, one cannot easily exploit the sparsity-promoting structure of $\Omega$ for computational gains.

In this paper, we address the non-sparseness problem (\ie no support identification) of online algorithms by combining them with the idea of safe screening rules~\cite{ghaoui2010safe,ndiaye2017gap}. More precisely, our contributions are as follows.
\begin{enumerate}
\item[{\rm (i)}] By adapting gap-safe screening rules of~\cite{ndiaye2017gap} to online algorithms, we propose an online-screening rule. The proposed rule only needs to evaluate function values at the sampled data, hence has low per iteration complexity. 
In particular, we show how to construct a ``dual certificate'' along the iterations which allows us to apply gap-safe rules to screen out certain features. 
Moreover, this certificate can be built alongside any convergent online algorithms. 

\item[{\rm (ii)}] The consequence of screening rules for online optimization is support identification of $\beta_t$, \ie dimension reduction, which allows us to locate features of interests. 
More importantly, significant computational gains can be obtained as per iteration complexity scales from $n$ (the dimension of the variable $\beta$) to $\kappa$ (the sparsity of the solution~$\beta^\star$). 

\end{enumerate}

\begin{remark}
An interesting feature of many batched optimization methods such as proximal gradient descent~\cite{lewis2016proximal,lewis2002active,liang2017activity} and coordinate descent~\cite{klopfenstein2020model}, is that they exhibit ``finite activity identification", where after a finite number of iterations, all iterates will have the same sparsity structure as the solution. Although one cannot check a-priori whether activity identification has been achieved, one can heuristically exploit this property for computational gains by switching to higher order methods once the support is sufficiently small.

\finalrev{When the optimization problem has a finite sums structure \eqref{eq:finitesum}, one can consider variance-reduced stochastic methods, such as proximal-SAGA and proximal-SVRG. These methods also enjoy   finite activity identification properties ~\cite{poon2018local} and  just as for the deterministic methods,  one can also heuristically exploit this property for computational gains. Moreover,  screening rules (such as Algorithm \ref{algo:full-screen}) can be applied and this leads to substantial performance gains in practice.} 

One the other hand, while activity identification is not present in many online algorithms (with the exception of regularized dual averaging~\cite{xiao2010dual}), we make use of screening in this work to enforce such a property. This idea is inspired by the recent work of~\cite{sun2020safe}, where the authors developed a gap-safe rule for conditional gradient descent. One highlight of their work is that through safe screening, identification is achieved whereas simply running conditional gradient descent will never do. In this work, we combine this idea with the gap-safe rules~\cite{ndiaye2017gap} to tackle the stochastic setting, where we dynamically construct safe regions by computing an ``online'' duality gap. 
\end{remark}

\paragraph{Paper organization}
The rest of the paper is organized as follows. 
We recall the basics of screening rules for sparsity-promoting regression problems in Section \ref{sec:screening}. Theoretical analysis of our online-screening rule is presented in Section \ref{sec:online}. 
Numerical experiments on LASSO and sparse logistic regression problems are provided in Section \ref{sec:experiments}. Finally, in the appendix we collect some basics of convex analysis and the proofs of the main theorems.

\section{Safe screening for sparse regularization}
\label{sec:screening}

Before introducing our algorithm, we first provide some background on screening rules, with particular focus on the gap-safe rule from~\cite{ndiaye2017gap}. 
Given a sparsity-promoting norm $\Omega$, its dual norm and sub-differential are respectively defined by
\[
\begin{aligned}
\Omega^D(Z) &\eqdef \sup\nolimits_{\Omega(\beta) \leq 1} \dotp{\beta}{Z} \qandq 
\partial \Omega(\beta) \eqdef \enscond{Z}{\dotp{Z}{\beta} = \Omega(\beta), \enskip \Omega^D(Z) \leq 1} .
\end{aligned}
\]
Note that if $\Omega$ is group decomposable, then we have $\Omega^D(Z) = \sup_{g\in\Gg}\Omega_g^D(Z).$

\subsection{Safe screening}
Below we summarize a few key facts about the support of solutions to \eqref{eq:finitesum}, and refer to~\cite{hastie2015statistical,vaiter2015model} for further details.
Let $\beta^\star$ be a global minimizer of the regularized empirical loss \eqref{eq:finitesum}, the first-order optimality condition entails
\beq\label{eq:opt-cnd}
0 \in \nabla F_\eta(\beta^\star) + \lambda \partial \Omega(\beta^\star) .
\eeq
This is equivalent to saying that $\beta^\star$ is a minimizer if and only if
\begin{equation}\label{eq:Zstarm}
\zmstar \eqdef - \qfrac{1}{\lambda } \msum_{i=1}^m \eta_i \theta_i^\star x_i 
\in \partial \Omega(\beta^\star)
\enskip\mathrm{with}\enskip
\theta_i^\star = \nabla f_i(x_i^\top \beta^\star, y_i),~ \; i=1,...,m .
\end{equation}
The fundamental idea behind screening rule comes from the above optimality condition, that given any sub-group $g \in \Gg$, we have
\beq\label{eq:one_direction}
\Omega_g^D(Z_g^\star) < 1 
\quad\Longrightarrow\quad
\beta_g^\star = 0 .
\eeq
The converse is also true under the non-degeneracy condition 
$
0 \in \ri \Pa{ \nabla F_\eta(\beta^\star) + \lambda \partial \Omega(\beta^\star) } $, where $\ri(\cdot)$ denotes the relative interior~\cite{hastie2015statistical,vaiter2015model}.

Take the $\ell_1$-norm for example, \ie $\Omega(\beta) = \norm{\beta}_1$, this means that $\lambda^{-1}\nabla F_\eta(\beta^\star)$ takes value $\pm 1$ only on the support of $\beta^\star$.
Moreover, $\theta^\star$ in \eqref{eq:Zstarm} is precisely the solution of the {\em dual problem} of \eqref{eq:finitesum}
\begin{equation}\label{eq:finitesumdual}\tag{$D_{\lambda,\eta}$}
\max_{\theta\in\Kk_{\lambda,\eta}} \bBa{ D_{\lambda,\eta}(\theta) \eqdef -\msum_{i=1}^m \eta_i f_{y_{i}}^*( \theta_i) }
\end{equation}
where $$\Kk_{\lambda,\eta} \eqdef \enscond{\theta}{\Omega^D\Pa{ \ssum_{i=1}^m \theta_i \eta_i x_{i} } \leq \lambda} \subset \bbR^{m}$$ is the dual constraint set. Since the vector $Z^\star$ certifies the support of $\beta^\star$, it is hence called the \textit{dual certificate}. 
The above message implies that, if $\zmstar$ is known, we can identify an index set $\calI$ which includes the support of the solution, that is 
\[
\mathrm{supp}(\beta^\star) \subseteq \calI \eqdef \enscond{g\in\Gg}{\Omega_g^D(\zmstar) = 1} . 
\] 
Consequently, one can restrict to optimization over $\beta_\calI \in \RR^{\abs{\calI}}$ instead. This can lead to huge computation gains if $\calI$ tightly estimates the true support $\mathrm{supp}(\beta^\star)$ which very often is much smaller than the dimension of the problem.

In general, computing $\theta^\star$ (hence $\zmstar$) is generally as difficult as finding $\beta^\star$. However, the  entries where $\zmstar$ saturates (takes values $\pm 1$) can be estimated more readily. 
This is exactly the idea of safe screening, which constructs a ``safe region" $\Rr_{\theta}$ such that $\theta^\star \in \Rr_{\theta}$. Then instead of using \eqref{eq:one_direction} to determine the zero entries of $\beta^\star$, one can consider the relaxed criteria: first let $\Zz \eqdef \Ba{ Z ~|~ Z = - \frac{1}{\lambda } \sum_{i=1}^m \eta_i \theta_i x_i,~\forall \theta \in \Rr_{\theta} }$, then $\beta_g^\star = 0$ if 
$\sup_{Z\in \Zz} \Omega_g^D(Z) < 1 .$
For the rest of the paper, we shall call $\Zz$ as the safe region. 
The following result, which can be found in~\cite{ghaoui2010safe}, illustrates how to perform screening rules based on a safe region $\Zz$. Let $\theta_{c}$ be the center of $\Rr_{\theta}$ and $Z_c = - \frac{1}{\lambda } \sum_{i=1}^m \eta_i \theta_{c, i} x_i \in \Zz$.

\begin{proposition}[Safe screen rule]\label{prop:rule}
Let $\beta^\star$ be a minimizer to \eqref{eq:finitesum} and suppose that 
\[
\zmstar \in \Zz\eqdef \enscond{Z}{\Omega^D_g(Z-Z_c) \leq r_g, \quad g\in\Gg}.
\]
Then, $\beta_g^\star = 0$ if $1 - \Omega_g^D(Z_c) > r_g$. 
\end{proposition}
%

{\noindent}There are several ground rules for constructing a safe region $\Zz$:
\begin{itemize}
\item The supremum of the dual norm over the safe region, \ie $\sup_{Z\in \Zz} \Omega_g^D(Z)$, is easy to compute. 
\item The size of the safe region should be as small as possible: as the most trivial safe region is the whole space which screens out nothing, while the best one is $\Zz = \{\zmstar\}$ which screens out all useless features. 
\end{itemize}

In the literature, various safe regions have been proposed. The very first safe screening work by~\cite{ghaoui2010safe} introduced the idea of static screening and sequential screening. For {static safe screening}, screening is only implemented as a pre-processing of data, hence it is crucial to construct a good safe region such that the amount of discarded features is as many as possible. If we have a finite sequence of regularization parameter $\lambda_{j}$ for $ j=0, ..., J$ such that $\lambda_0 \geq \lambda_1 \geq \cdots \geq \lambda_J = \lambda $. Then static screening can be applied to each $\lambda_j$ which results in sequential screening. For both static and sequential screening, the volume of the safe regions is always bounded way from $0$ which limits the potential of screening. In addition to the safe region proposed in~\cite{ghaoui2010safe}, other safe regions include dual polytope projection safe sphere~\cite{wang2013lasso} and { safe dome}~\cite{xiang2016screening}. Dynamic screening rules were later proposed in~\cite{ndiaye2015gap,ndiaye2017gap}, where they combine screening rules and numerical methods such that the constructed safe regions are generated by the sequence of the numerical scheme. As a result, the safe region can eventually converge to the dual certificate and screen out all useless features. Our approach will follow the idea of dynamic screening.

\subsubsection{Gap-safe screening}

\renewcommand{\algorithmiccomment}[1]{{\color{blue}//#1}}

\begin{algorithm}[!h]
\SetAlgoLined
Given: $T>0$, step-size $\ens{\gamma_t}_{t\in\NN}$\;
 initialization $t=1$, $\bar{\beta}_0\in\RR^n$\; 
 \While{not terminate}{
 $\beta_{0} = \bar{\beta}_{t-1}$ \tcp*{\tcb{set an anchor point}}

	\For{$j=0,\ldots, T-1$}{
 		Sample $(x_j, y_j)$ \rev{$\sim \Lambda$} \tcp*{\tcb{random sampling}} 
 		$\beta_{j} = \Tt(\beta_{j-1}, f_{y_j}'( x_j^\top \beta_{j-1}) x_j, \gamma_j)$ \tcp*{\tcb{standard gradient update}}
 		$j=j+1$;
 }
 $\bar{\beta}_t = \beta_T, \bar{\theta}_t = ( f_{y_i}'(x_i^\top \bar{\beta}_t))_{i=1}^{m}$ \tcp*{\tcb{primal and dual variables}}
 Compute safe centre $Z$ and radius $r_g$ \tcp*{\tcb{e.g. as in \eqref{eq:gap-certificate} and \eqref{eq:radius} }}
 
	$\Ss = \enscond{g\in\Gg}{\Omega^D_g(Z) < 1- r_g }$ \tcp*{\tcb{screening set}}
	$(\bar{\beta}_t)_{\Ss} = 0$ \tcp*{\tcb{pruning the primal point}}
	 $t = t + 1$\;
 }
 \caption{Safe screening for finite sum problem \label{algo:full-screen}}
\end{algorithm}

In a series of work~\cite{ndiaye2015gap,ndiaye2016gap,ndiaye2017gap}, the authors develop a gap-safe rule for screening, where the ``gap'' refers to the duality gap between the primal function \eqref{eq:finitesum} and dual function \eqref{eq:finitesumdual}. 
For any $\beta \in \bbR^n$ and $\theta \in \Kk_{\lambda,\eta}$, the duality gap is defined by
\[
\begin{aligned}
G_{\lambda,\eta}(\beta,\theta)
&= P_{\lambda,\eta}(\beta) - D_{\lambda,\eta}(\theta) .
\end{aligned}
\]
Let $\beta^\star$ and $\theta^\star$ be a primal and dual solution respectively, then strong duality holds and
\[
D_{\lambda,\eta}(\theta) \leq D_{\lambda,\eta}(\theta^\star)
=
P_{\lambda,\eta}(\beta^\star) \leq P_{\lambda,\eta}(\beta) ,\quad \forall \beta\in\RR^n,~~\theta\in\RR^m . 
\]
As a result, the duality gap $G_{\lambda,\eta}(\beta,\theta)$ is always non-negative.

Since the loss function is differentiable with $L$-Lipschitz continuous gradient, the dual problem $D_{\lambda,\eta}(\theta)$ is $\mu$-strongly concave with $\mu = 1/L$. Then for any $\beta \in \RR^n$ and $\theta \in \Kk_{\lambda,\eta}$, 
$
\sfrac{\rev{\mu \lambda^2}}{2}\norm{\theta-\theta^\star}^2
\leq P_{\lambda,\eta}(\beta) - D_{\lambda,\eta}(\theta) 
$~\citep[Theorem 6]{ndiaye2017gap}.
Therefore, letting 
\begin{equation}\label{eq:radius}
r_t \eqdef \rev{ \sqrt{2G_{\lambda,\eta}(\beta_t, \theta_t)/(\mu\lambda^2)}},
\end{equation}
 one obtains the following safe sphere: 
\[
\Zz \eqdef \bBa{ - \sfrac{1}{\lambda } \msum_{i=1}^m \eta_i \theta_i^\star x_i , \enskip \forall \theta \in \Rr_{\theta} } 
\qwithq
\Rr_{\theta} \eqdef \Ba{ \theta \in \RR^m : \norm{\theta - \theta_t} \leq r_t } .
\]
Now given a numerical scheme, at each iteration, $\beta_t$ is explicitly available and one can compute a dual feasible variable $\theta_t$ by projecting $\bar \theta_t \eqdef (f'_{y_i}(x_i^\top \beta_t))_{i=1}^{m}$ on to the dual feasible set $\Kk_{\lambda,\eta}$.
In particular, define
\begin{equation}\label{eq:gap-certificate}
Z\eqdef c^{-1} \msum_i \eta_i (\bar \theta_t)_i x_i 
\qwhereq 
c \eqdef \max \Big\{ 1, \Omega^D\Pa{\sfrac{1}{\lambda} \msum_{i=1}^m \eta_i (\bar \theta_t)_i x_i} \Big\}.
\end{equation}
It follows by using Holder's inequality and the fact that $\Omega^D$ is positive homogeneous that the distance from  $Z$  to the true dual solution $Z^\star$ (see \eqref{eq:Zstarm}) is bounded by
\[
\forall g\in \Gg,\quad
\Omega_g^D(Z-Z^\star) \leq \norm{\theta_t-\theta^\star} \sqrt{\msum_i \eta_i^2 \Omega_g^D(x_{i,g})^2}.
\]
Hence, $(\beta_t)_g = 0$ if $\Omega_g^D(Z) < 1- r_t\sqrt{\sum_i \eta_i^2 \Omega_g^D(x_{i,g})^2} \eqdef 1-r_g$.
\subsection{Gap-safe screening for Prox-SGD}

 Screening rules are algorithm agnostic. That is to say, given an algorithm with iterates $\beta_t$, one can always compute a safe region $\calZ$ for screening. 
As a result, we can incorporate screening to proximal stochastic gradient descent when the problem to solve has a finite sum empirical loss of the form \eqref{eq:finitesum}. This is the most straightforward way of carrying out screening rule, and indeed, a similar screening strategy was recently proposed for ordered weighted $\ell_1$-norm regularized regression in~\cite{bao2020fast}.

For the finite sum problem, consider an algorithm of the following general form: for each $t$, sample $(x_t, y_t)$ uniformly at random from the finite data
\begin{equation}\label{eq:sgd_finite}
\begin{aligned}
\beta_{t+1} = \Tt(\beta_t, \theta_t x_t, \gamma_t) \quad \textrm{with} \quad \theta_t = f_{y_{t}}'(x_{t}^\top \beta_t) . 
\end{aligned}
\end{equation}
We have for SGD, $\Tt(\beta_t,\phi_t,\gamma_t) = \beta_t - \gamma_t( \phi_t + \lambda Z_t)$ with $Z_t\in \partial \Omega(\beta_t)$, while for Prox-SGD $\Tt(\beta_t,\phi_t,\gamma_t) = \prox_{\gamma_t \lambda \Omega}\pa{\beta_t - \gamma_t \phi_t}$. 
Algorithm \ref{algo:full-screen}  combines~\eqref{eq:sgd_finite} with safe screening rules.

\begin{remark}
Algorithm \ref{algo:full-screen} has two loops of iterations: the inner loop is the standard stochastic gradient update, while for the outer loop is screening with certain safe rules. Note that the outer loop makes use of $\bar \theta_t$ which is evaluated over the entire dataset.
Such a setting is reminiscent of the SVRG algorithm~\cite{johnson2013accelerating}, where the full gradient of the loss function at an anchor point needs to be computed. 
Likewise, the choice of steps for inner loop, the value of $T$ in Algorithm \ref{algo:full-screen} should be of the order of $m$, to balance the overhead of computing $\theta_t$.
\end{remark}

\begin{remark}
For Algorithm \ref{algo:full-screen}, all the aforementioned safe screening rules can be applied; see~\cite{ghaoui2010safe,liu2014safe,ndiaye2015gap,ndiaye2017gap} and the references therein. 
However, for online learning, this is no longer true, since for online learning it is expensive or even impossible to obtain the projected point $\bar\theta_t$, let alone construct the safe region $\calZ$ using \eqref{eq:gap-certificate}. Hence, in what follows, we propose an approach to compute an online gap and construct a safe region without the need to project onto the constraint set.
\end{remark}

\section{Screening for online algorithms}\label{sec:online}

For large-scale problems online optimization methods, it is unrealistic or impossible to compute the dual variable $\bar \theta_t$. Consequently, one cannot construct the safe region $\calZ$ for screening. 
However, for the gap-safe screening rule, since its safe region is built on function duality gap, it is possible to generalize the rule to the online setting via stochastic approximations. The purpose of this section is to build such a generalization. The roadmap of this section is described below: 
\begin{enumerate}
\item We first describe how to construct online dual certificates and primal/dual objectives, which consist of the following aspects: a) the dual problem of the online problem \eqref{eq:onlineprob}; b) online duality gap via stochastic approximations; c) online dual certificate; d) convergence guarantees. These are provided in Section \ref{sec:step1}. 

\item {With the online duality gap and dual certificate obtained in the first stage Section \ref{sec:step1}, we then can extend the gap-safe screening rule to the online setting. This extension is described in Section \ref{sec:onlinescreen}.} %

\item Finally in Section \ref{sec:procedure}, in Algorithm \ref{algo:screen} we summarize our online-screening scheme for proximal stochastic gradient descent. 
\end{enumerate}

\subsection{Online dual certificates and objectives}\label{sec:step1}

Given an online method, at each step we sample $(x_t, y_t)$ from the distribution $\Lambda$ and evaluate
\begin{equation}\label{eq:theta_t}
\theta_t \eqdef f_{y_t}'(x_{t}^\top \beta_{t}) . 
\end{equation}
In what follows, we define an online dual point $\bar Z_t$ that is constructed as weighted average of the past evaluated points $\ens{\theta_{s}}_{s\leq t}$ and define online primal and dual objectives that are again weighted averages of the past selected functions $\ens{f_{y_{s}}}_{s\leq t}$ and $\ens{f^*_{y_{s}}}_{s\leq t}$, where $f_y^*$ denotes the convex conjugate of loss function $f_y$.

\subsubsection{Online dual problem and duality gap}
The dual problem of the primal problem \eqref{eq:onlineprob} takes the following form
\begin{equation}\label{eq:onlinedual}\tag{$D_{\lambda}$}
\max_{v}\enscond{ \Dd(v) \eqdef - \EE_{(x,y)}\big[f_y^*\Pa{ v(x,y)} \big]}{ \Omega^D\Pa{\EE_{(x,y)}[v(x,y) x ]} \leq \lambda}
\end{equation}
where we maximize over $\Lambda$-measurable functions $v$. \rev{The derivation of \eqref{eq:onlinedual} can be found in Appendix \ref{sec:onlinedual}.}
Note that it admits a unique maximizer, since $f_y^*$ is $\frac1L$-strongly convex  due to the fact that $\nabla f$ is $L$-Lipschitz. 
The problems \eqref{eq:onlineprob} and \eqref{eq:onlinedual} are referred as primal and dual problems and their solutions are related:
any minimizer $\beta^\star$ of \eqref{eq:onlineprob} is related to the optimal solution $v^\star$ of \eqref{eq:onlinedual} by $v^\star(x,y) = f'(x^\top \beta^\star, y) $ and
\begin{equation}\label{eq:certstar}
 \cert^\star \eqdef - \sfrac{1}{\lambda} \EE_{(x,y)} [v^\star(x,y) x] \in \partial \Omega(\beta^\star).
\end{equation}

\renewcommand{\NN}{\mathbb{N}}

Observe that the primal and dual objective functions are expectations. We now discuss their online ergodic estimations over the sampled data. 
At time step $t\in\NN$, given a primal variable $\beta\in\RR^n$, define the online primal objective
\begin{equation}
\bar P^{(t)}(\beta) \eqdef \bar F^{(t)}(\beta) +\lambda \Omega(\beta)
\end{equation} 
where $\mu_t\in (0,1)$, $\bar F^{(1)}(\beta) = f_{y_1}(x_{1}^\top \beta)$ and
\[
\bar F^{(t)}(\beta) \eqdef \mu_t f_{y_{t}}(x_{t}^\top \beta) + (1-\mu_t) \bar F^{(t-1)}(\beta) .
\]
For each step $s \leq t$, $\theta_s = f_{y_s}'(x_{s}^\top \beta) \in \RR$ denotes the dual variable of that step. Let $\theta = (\theta_1, ..., \theta_s, ..., \theta_t) \in \RR^t$, the online dual objective for $\theta$ reads: let $\bar D^{(1)}(\theta_1) = - f_{y_{1}}^*( \theta_1)$
\begin{equation}\label{eq:barDt}
\bar D^{(t)}(\theta) \eqdef -\mu_t f_{y_{t}}^*( \theta_t) + (1-\mu_t) \bar D^{(t-1)} \Pa{ (\theta_s)_{s\leq t-1} } . 
\end{equation}

\begin{remark}
Note that the primal variable has a fixed dimension of $n$, while the dimension of the dual variable $\theta$ grows with iteration $t$. 
\end{remark}

We make the following standard assumption \citep{robbins1951stochastic} on $\mu_t$:
\begin{equation}\label{eq:mu_assump}
\msum_t \mu_t = +\infty
\qandq 
\msum_t \mu_t^2 < + \infty . 
\end{equation}
 Typical choices are $\mu_t = t^{-u}$ for $u\in (0.5,1]$.

It is straightforward to check (Lemma \ref{lem:mairal} (i)) that there exists a decreasing sequence $\eta_s^{(t)}>0$ such that
\begin{equation}\label{eq:eta_mu}
 \msum_{s\leq t} {\eta_s^{(t)}} = 1 \qandq \eta_s^{(t)} \eqdef \mu_s \mprod_{i=s+1}^{t} (1-\mu_{i}) .
 \end{equation}
Using $(\eta_s^t)_{s\leq t}$, in our previous notation of empirical \eqref{eq:finitesum} and \eqref{eq:finitesumdual}, we have $\bar P^{(t)} = P^{\eta^{(t)}}$ and $\bar D^{(t)} = D^{\eta^{(t)}}$, and they are related by
\[
\min_{\beta\in \RR^n}\bar P^{(t)}(\beta ) = \max_{\theta\in \Kk_{\lambda,\eta^{(t)}}} \bar D^{(t)}(\theta) . 
\]

\begin{definition}[Online duality gap]
Let $\theta_s$ be as in \eqref{eq:theta_t} and $\beta\in\RR^n$, define the online duality gap as
\begin{equation}\label{eq:gap}
\gap_t(\beta) \eqdef \bar P^{(t)}(\beta) - \bar D^{(t)}((\theta_s)_{s\leq t}).
\end{equation}
\end{definition}

{
\begin{remark}
Since $\theta_s$ is not necessarily a dual feasible point, the online duality gap $\gap_t(\beta)$ is not guaranteed to be non-negative. 
On the other hand, while $\gap_t(\beta)$ can be computed in an online fashion, the feasible point $\bar \theta_s$, in $\bar P^{(t)}(\beta) - \bar D^{(t)}((\bar \theta_s)_{s\leq t})$, which is the projection of $\theta_s$ onto the constraint set $\Kk_{\lambda,\eta^{(t)}}$ cannot be computed online.
\end{remark}
}

\subsubsection{An online estimate of the dual certificate}

With the online duality gap, we now construct a dual certificate from $\beta_t$ and $\theta_t$. 
Since the primal variable $\beta_t$ converges to $\beta^\star$, it is natural to define a candidate point, for $\mu_t>0$, as
\begin{equation}
\label{eq:Zstar}
\bcert \eqdef -\sfrac{1}{\lambda} \mu_t \theta_t x_{t} + (1-\mu_t) \bar{\cert}_{t-1}, \qandq \bar\cert_1 \eqdef -\sfrac{1}{\lambda} \theta_1 x_{1}.
\end{equation}
In the notation introduced in \eqref{eq:eta_mu}, 
we can write
\begin{equation}\label{eq:onlinecert}
\bcert 
= -\sfrac{1}{\lambda} \msum_{s=1}^t \eta_s^{(t)} \theta_s x_{s} 
= -\sfrac{1}{\lambda} \msum_{s=1}^t \eta_s^{(t)} f_{y_s}'(x_{s}^\top \beta_{s}) x_{s} .
\end{equation}

\subsubsection{Convergence results}

Before presenting our online-screening rule, we provide some theoretical convergence analysis of the above online estimates. 
\rev{To this end, we need the following assumptions
\begin{enumerate}[label={\rm{*}}]
\item[{\rm (i)}] let $\Oo, \Yy$ be compact sets, and assume there exists $G>0$ such that: \[\abs{f_y'(x^\top\beta)} \leq G, \quad \forall~ (x,y) \in\Xx\times \Yy \qandq \beta\in\Oo . \] 

\item[{\rm (ii)}] assume that $\beta_t$ converges to a minimizer $\beta^\star\in \Argmin_\beta P_{\lambda}(\beta)$ with $\beta^\star \in\Oo$.
\end{enumerate}
Define
 \[
 \epsilon_t\eqdef n
\sqrt{ \msum_{j=1}^t (\eta_j^{(t)} )^2}\lesssim n \sqrt{ \min_{m<t} { \msum_{j=m}^t \mu_j^2 + \exp(-2\msum_{j=m}^t \mu_j)} }.
 \]
Under assumption \eqref{eq:mu_assump}, one can show (see Lemma \ref{lem:mairal}) that $\epsilon_t \to 0$ as $t\to\infty$ and
 if $\mu_j = 1/j^\gamma$ with $\gamma\in (0.5,1]$, then $\epsilon_t = \Oo(t^{-\gamma+\frac12})$.}

We first establish uniform convergence of the online objective $\bar P^{(t)}$ to its expectation $\Pp$, and uniform convergence of the corresponding dual certificate $Z^{\star,(t)}$ to $Z^\star$.

\begin{proposition}
\label{prop:obj_conv}

The following result holds

\begin{enumerate}[label={\rm{*}}]
\item [{\rm (i)}] Let $\Oo$ be a compact set of $\RR^n$, then almost surely 
$
\lim_{t\to+\infty}\sup_{\beta \in\Oo} \abs{ \bar P^{(t)}(\beta) - \Pp(\beta) } = 0 $. 

\item[{\rm (ii)}] Let $Z^{\star,(t)}$ be the dual certificate associated to $\bar P^{(t)}$, we have uniform convergence to $Z^\star$ defined in \eqref{eq:certstar}: 
\[
\lim_{t\to+\infty} \norm{Z^{\star,(t)}- Z^\star}_\infty = 0 
\rev{\qandq \EE[\norm{Z^{\star,(t)}- Z^\star }_\infty ] = \Oo\pa{ \sqrt{\epsilon_t}} , }
\] 
where the implicit constant in the Big-$\Oo$ depends on $n$, the dimension of $\beta$ (it comes from the equivalence between norms in finite dimensions, $\Omega^D$ and  $\norm{\cdot}_\infty$).
\end{enumerate}

\end{proposition}

We also have convergence of the online certificate $\bcert$ of \eqref{eq:onlinecert} to $\cert^\star$, and the online gap evaluated at converging points also converges to zero. 

\begin{proposition}[Convergence of the online estimate] \label{prop:dualconv}
Let $\beta^\star$ be a minimizer of \eqref{eq:onlineprob}, 
assume $\beta_t \to \beta^\star$ almost surely and \rev{$\EE[\norm{\beta_t - \beta^\star}] \to 0$}.
 Then,   $\lim_{t\to\infty}\msum_s \eta_s^{(t)} \EE\norm{\beta_s - \beta^\star} = 0$ and the following hold:
\begin{itemize}
\item[{\rm (i)}] $\bcert$ converge to $\cert^\star$ as $t\to+\infty$ almost surely and \begin{align*}
\EE\norm{ \bcert- Z^\star }_\infty & \lesssim \epsilon_t + \msum_s \eta_s^{(t)} \EE\norm{\beta_s - \beta^\star} .
\end{align*}
\item[{\rm (ii)}] If $\bar \beta_t \to \beta^\star$, then $\gap_t(\bar \beta_t) \to 0$ as $t\to+\infty$ almost surely and
 \begin{align*}
\EE\abs{\gap_t(\bar \beta_t)} & \lesssim \epsilon_t + \EE\norm{\bar \beta_t - \beta^\star} + \msum_s \eta_s^{(t)} \EE\norm{\beta_s - \beta^\star} . 
\end{align*} 
\end{itemize}

\end{proposition}
The proofs can be found in the Appendix \ref{proof:section3}.

\subsection{Online-screening}\label{sec:onlinescreen}

In this section, we derive a screening rule for solutions to the online objective $\bar P^{(t)}$ based on the certificate $\bcert$ of \eqref{eq:onlinecert} and $\gap_t(\beta)$ of \eqref{eq:gap}. 
In the following, let $\bcert$ be as in \eqref{eq:onlinecert}, $\theta_t$ be as in \eqref{eq:theta_t} and let $Z^{\star,(t)}=-\frac{1}{\lambda} \sum_{s=1}^t \eta^{(t)}_s \theta^{\star, (t)}_s x_{s}$ where $\theta^{\star, (t)}$ is the maximizer of $\bar D^{(t)}$~\eqref{eq:barDt}.

\begin{lemma}[Screen gap] \label{lem:radius_gap}
Let $\bar\beta\in\RR^n$ \finalrev{and $\hat \beta\in \RR^n$}, then there holds
\[
\sfrac{1}{2L} \msum_{s=1}^t \eta^{(t)}_s \abs{ \theta_{s} - \theta^{\star, (t)}_s}^2 
\leq \gap_t(\bar\beta) + \finalrev{\bar P^{(t)}(\hat \beta)} \; \Pa{ \Omega^D(\bcert)-1 }_+ 
\]
Moreover, for all $g\in\Gg$,
\begin{align*}
\Omega_g^D\pa{\bcert - Z^{\star,(t)} } \leq r_g^{(t)}(\bar\beta, \bcert) 
\end{align*}
where,  $N_g\eqdef { \sum_{s=1}^t \eta^{(t)}_s \Omega_g^D( x_{s})^2 }$ and
\[
r_g^{(t)}(\bar\beta, \bcert) \eqdef \sfrac{\sqrt{2L N_g }}{\lambda} \; \sqrt{ \gap_t(\bar\beta) + \finalrev{\bar P^{(t)}(\hat \beta)}\; \pa{\Omega^D(\bcert) -1}_+} . 
\]

\end{lemma}
\finalrev{
\begin{remark}
The above holds for any $\bar \beta$ and $\hat \beta$. In our numerics, we choose $\hat \beta = \bar \beta$  where $\bar \beta$ is an ``anchor point" which is updated periodically (see Algorithm \ref{algo:screen}).
\end{remark}
}

\begin{remark}
Compared with the gap-safe rules of~\cite{ndiaye2015gap,ndiaye2016gap,ndiaye2017gap}, we do not project $\bcert$ onto the dual feasible set $\Kk_{\lambda, \eta^{(t)}}$, hence having an additional term $\pa{\Omega^D(\bcert) -1}_+$ in $r_g^{(t)}(\beta,\bcert)$.

\end{remark}

By combining Lemma \ref{lem:radius_gap} with Proposition \ref{prop:rule}, we obtain the following proposed screening rule for online optimization algorithm.

\begin{corollary}[Online screen rule]\label{cor:rule}
Let $\beta^{\star,(t)}\in \Argmin_\beta \bar P^{(t)}(\beta)$. Then,
given any $\bar \beta\in\RR^n$, $\beta^{\star,(t)}_g = 0$ if
\[
1- \Omega_g^D(\bcert) > r_g^{(t)}(\bar \beta, \bcert) .
\]
\end{corollary}

\begin{remark}
The above screening is safe for the online problem $\bar P^{(t)}(\beta)$ in the sense that it will not falsely remove features which are in the solution $\beta^{\star,(t)}$ of $\bar P^{(t)}(\beta)$. However, the support of $\beta^{\star,(t)}$ may not necessarily coincide with that of the global minimizer $\beta^\star$ of \eqref{eq:onlineprob}. 
Hence, our rule is not necessarily safe for the objective expectation \eqref{eq:onlineprob}. Further discussions on the safety of our rule can be found in Sections \ref{sec:section_checks} and \ref{sec:safety_check}. 
\end{remark}

\paragraph{A sequential screening strategy}

We can directly apply Corollary \ref{cor:rule} 
to screen out variables while running SGD. However, the effectiveness of this rule will depend on the proximity of $\bar \beta$ to the optimal point $\beta^\star$. We therefore propose to progressively update this \emph{anchor point} $\bar \beta$.


Let $0=t_0 < t_1 < t_2 <\cdots < t_k = T$ and denote $[t_{j-1}, t_j]\eqdef \ens{t_{j-1}+1,\ldots, t_j}$. Let $\eta^{(T)}_{s}\in (0,1)$ for $s\in [0,T]$ be such that $\sum_{s\in[0,T]}\eta^{(T)}_{s}=1$.
 Given $j\in \ens{1,\ldots, k}$, let $\pa{\theta_s^\star}_{s \in [t_{j-1}, t_j]}$ be the optimal dual solution to
{
\begin{equation}\label{eq:d1}
\begin{aligned}
& \max_\theta \msum_{s\in [t_{j-1}, t_j]} -\eta_s^{(T)} f_{y_s}^*(\theta_s) \\
\textrm{such~~that} \quad & \Omega^D \bPa{\msum_{s \in [t_{j-1}, t_j]} \eta_s^{(T)} x_s \zeta_s} 
\leq \lambda \msum_{s \in [t_{j-1}, t_j]} \eta_s^{(T)} . 
\end{aligned}
\end{equation}
}
Note that \eqref{eq:d1} 
is dual to the primal problem 
\[
\min_\beta \msum_{s \in [ t_{j-1},t_j]} \eta_s^{(T)} \Pa{ f_{y_s}(x_s^\top \beta) + \lambda\Omega(\beta)}.
 \]
 The corresponding dual certificate of \eqref{eq:d1} is
\[
Z_j^\star 
\eqdef - \sfrac{1}{\lambda} \msum_{s \in [ t_{j-1},t_j]} \sfrac{\eta_s^{(t)} }{\gamma_j}\theta_s^\star, 
\qwhereq \gamma_j = \msum_{s \in [ t_{j-1},t_j]} \eta_s^{(T)}. 
\]
For each $j$, applying Lemma \ref{lem:radius_gap} with $\ens{\theta_s}_{s\in [t_{j-1},t_j]}$, $\hat \beta =\bar \beta = \beta_{t_{j-1}}$ and $ \ens{\theta_s^\star}_{s\in [t_{j-1},t_j]}$, we get
\finalrev{
\begin{equation}\label{eq:gapj}
\begin{split}
&\msum_{s\in [t_{j-1},t_j]} \eta_s^{(T)} \abs{\theta_s - \theta_s^\star}^2 \\
&\leq \msum_{s\in [t_{j-1},t_j]} \eta_s^{(T)} \big( (f_{y_s}(x_s^\top \beta_{t_{j-1}}) + \lambda \Omega(\beta_{t_{j-1}})  ) (1+ \Pa{\Omega^D (Y_j) - 1}_+)  - f_{y_s}^*(\theta_s)  \big)
\end{split}
\end{equation}
}
where
$
Y_j \eqdef \frac{1}{\sum_{s\in [t_{j-1}, t_j]} \eta_s^{(T)} } \sum_{s\in [t_{j-1}, t_j]} \eta_s^{(T)} \theta_s x_s .
$
Summing \eqref{eq:gapj} over $j=1,\ldots, k$ and denoting $\bar \beta_s \eqdef \beta_{t_{j-1}}$ for $s\in [t_{j-1}, t_j]$, we obtain
\begin{equation}\label{eq:sc1}
\msum_{s=1}^T \eta_s^{(T)} \abs{\theta_s - \theta_s^\star}^2 \leq R_T
\end{equation}
where 
\finalrev{
$$
R_T \eqdef \sum_{j=1}^k \msum_{s\in [t_{j-1},t_j]} \eta_s^{(T)} \Ppa{ f_{y_s}(x_s^\top \bar \beta_s) + \lambda \Omega(\bar \beta_s)} (1+\pa{\Omega^D (Y_j) - 1}_+) -  \sum_{s=1}^T \eta_s^{(T)} f_{y_s}^*(\theta_s)
$$
}
Lastly define
\[
\bar Z^\star 
= \msum_{j=1}^k \gamma_j Z_j^\star 
=- \sfrac{1}{\lambda} \msum_j \msum_{s \in [ t_{j-1},t_j]} \eta_s^{(T)} \f_{y_s}'(x_s^\top \bar \beta_j^\star) x_s , 
\]
we have $\bar Z_T \eqdef \sum_{s=1}^T \eta_s^{(T)} \theta_s x_s$ satisfying
\begin{equation}\label{eq:sc2}
\Omega_g^D(\bar Z_T - \bar Z^\star) 
\leq \sfrac{\sqrt{2L R_T N_{T,g} }}{\lambda} ,
\end{equation}
where
$N_{T,g} \eqdef { \sum_{s=1}^T \eta^{(T)}_s \Omega_g^D( x_{s})^2 }$. Note that the residual term $R_T$ now depends on the sequence $\beta_{t_j}$ which converges to $\beta^\star$ as $j\to+\infty$. We can therefore expect the RHS of \eqref{eq:sc2} to converge to 0 as $T\to+\infty$. Having established how to progressively update the anchor point, we are finally able to present our online-screening rule for online optimization algorithms in the next section.


\subsection{Screening procedure}\label{sec:procedure}

Consider an online algorithm of the following form: for $t=0,1,2,\ldots$, draw sample $(x_t,y_t) \sim \Lambda$, and compute
\begin{equation}\label{eq:sgd_}
\begin{split}
\beta_{t+1} = \Tt(\beta_t, \theta_t x_t, \gamma_t) \quad\textrm{with}\quad \theta_t = f_{y_t}'(x_t^\top \beta_t) ,
\end{split}
\end{equation}
where $\Tt$ is the algorithm operator, and again $\Tt(\beta_t,\phi_t,\gamma_t) = \beta_t - \gamma_t( \phi_t + \lambda Z_t)$ in the case of SGD, and $\Tt(\beta_t,\phi_t,\gamma_t) = \prox_{\gamma_t \tau\Omega}\pa{\beta_t - \gamma_t \phi_t}$ for Prox-SGD.
We state our screening framework for online optimization methods in Algorithm \ref{algo:screen}. 

\finalrev{
\begin{algorithm}[!htp]
\SetAlgoLined
Given: step-size $\ens{\gamma_t}_{t\in\bbN}$, exponent $w$, 		 		$\mu_t \eqdef 1/t^w$, initial point $\bar\beta\in\RR^n$\;
 initialization $t=1$; $p_0 =d_0=N_{0,g} = 0$\;
 $u_0 = 1$;
 $S=0$; $Z=0_n$\;
 \While{not terminate}{
 $\beta_{0} = \bar \beta$ \tcp*{\tcb{set the anchor point}}
 $p_0 = 0$; $X_0 = 0$\;

	\For{$t=1,\ldots, T$}{
 		$(x_t,y_t)\sim \Lambda$ \tcp*{\tcb{random sampling}} 
 		$\theta_t = f_{y_t}'( x_t^\top \beta_{t-1})$; $\beta_{t} = \Tt(\beta_{t-1}, \theta_t x_t, \gamma_t)$ \tcp*{\tcb{standard gradient update}}

 		$ X_t = -\frac{1}{\lambda} \mu_t \theta_t x_t + (1-\mu_t) X_{t-1}$ \tcp*{\tcb{certificate update}}
 		$p_t = \mu_t\pa{f_{y_t}(x_t^\top \bar \beta) + \lambda \Omega(\bar \beta)} + (1-\mu_t) p_{t-1}$ \tcp*{\tcb{primal value}}
 		$d_t = - \mu_t f_{y_t}^*(\theta_t) + (1-\mu_t) d_{t-1}$ \tcp*{\tcb{dual value}}
 		$\forall g\in\Gg, \; N_{t,g} = \mu_s \Omega^D_g(x_t)^2 + (1-\mu_s) N_{t-1,g}$ \;
 		{$u_t = (1-\mu_t) u_{t-1}$}\;
 }
 $\bar\beta = \beta_{t-1}$ \tcp*{\tcb{update anchor point}}
 $Z = u_t Z + X_t$ \tcp*{\tcb{estimated certificate}}
 $S = u_t S + p_t (1+ ( \Omega^D(X_t/(1-u_t) )- 1)_+)$\;
 
 $R = S - d_t$\;
	$\Ss = \enscond{g\in\Gg}{\Omega^D_g(Z) < 1- \tfrac{ \sqrt{2L N_{t,g} R}}{\lambda} }$ \tcp*{\tcb{screening set}}
	$(\beta_t)_{\Ss} = 0$ \tcp*{\tcb{pruning the primal point}}
	 $u_0 = 1$\;
 }
 \caption{Online optimization algorithm with screening \label{algo:screen}}
\end{algorithm}
}

Next we provide some discussions on how to compute some key values of the algorithm, for instance the terms described in \eqref{eq:sc1} and \eqref{eq:sc2}.
\begin{itemize}
\item 
It is straightforward to compute $\bar Z_T$ and $N_{T,g}$, as we have $\bar Z_0 = 0,N_{0,g} = 0$ and 
\begin{align*}
\bar Z_s \eqdef \mu_s \theta_s x_s + (1-\mu_s) \bar Z_{s-1} \qandq 
N_{s,g} &\eqdef \mu_s \Omega^D_g(x_s)^2 + (1-\mu_s) N_{s-1,g} .
\end{align*}

\item \finalrev{ While for $R_T$, it takes the following form 
\[
R_T \eqdef \underbrace{\msum_{j=1}^k p_j (1+ \Pa{\Omega^D (Y_j) - 1}_+)}_{S_T}
- \underbrace{\msum_{s=1}^T - \eta_s^{(T)} f_{y_s}^*(\theta_s)}_{d_T}  . 
\]
where $p_j\eqdef  \sum_{s\in [t_{j-1},t_j]} \eta_s^{(T)}   ( f_{y_s}(x_s^\top \bar \beta_s)  +\lambda \Omega(\bar \beta_s))$.
The second term is straightforward: define $\bar \beta_s = \beta_{t_{j-1}}$ for all $s\in [t_{j-1}, t_j]$ and repeat  over $s=1,\ldots, T$: $ d_0 = 0$ and 
\begin{align*}
d_s \eqdef - \mu_s f_{y_s}^*(\theta_s) + (1-\mu_s) d_{s-1}  .
\end{align*}

To compute $S_T$:
during the first $s=1,\ldots, t_1$ iterations, let $X_0^{(1)} = 0, p_0^{(1)} = 0$ and 
\begin{align*}
X_s^{(1)} \eqdef \mu_s \theta_s x_s + (1-\mu_s) X_{s-1}^{(1)} \qandq
p_s^{(1)} \eqdef \mu_s (f_{y_s}(x_s^\top \bar \beta_s) +\la \Omega(\bar \beta_s))+ (1-\mu_s) p_{s-1}^{(1)}  .
\end{align*}
and note that
$
 Y_1 \eqdef X_{t_1},~ p_1 \eqdef p_{t_1}^{(1)},~ S_{t_1} \eqdef p_1 (1+\Pa{ \Omega^D(Y_1)-1 }_+)
$. 
Then for iteration $s \in [t_{j-1} , t_{j}],~ j= 2,~ \dots, k$, we have: $X_{t_{j-1}}^{(j)} = 0, p_{t_{j-1}}^{(j)} = 0$ and 
\begin{align*}
X_s^{(j)} &\eqdef \mu_s \theta_s x_s + (1-\mu_s) X_{s-1}^{(j)} \qandq 
p_s^{(j)} &\eqdef \mu_s  (f_{y_s}(x_s^\top \bar \beta_s) +\la \Omega(\bar \beta_s)) + (1-\mu_s) p_{s-1}^{(j)}   .
\end{align*}
At iteration $t_j$: define $\gamma_j \eqdef \prod_{s\in[t_{j-1}, t_j]}(1-\mu_{s})$ and
\begin{equation*}
Y_j \eqdef \sfrac{1}{1-\gamma_j} X_{t_j}^{(j)}, \qquad p_j \eqdef p_{t_j}^{(j)}
\qandq
S_{t_j} \eqdef \gamma_j S_{t_{j-1}} +  p_j (1+(\Omega^D(Y_j)-1)_+).
\end{equation*}
Note that we in fact have $\bar Z_{t_j} = X_{t_j}^{(j)} + \gamma_j \bar Z_{t_{j-1}}$.
}

\end{itemize}
We conclude this section by few remarks.

\begin{remark}$~$\label{rmk:w}
\begin{itemize}
\item [{\rm (i)}]
{\bf Computational pains and gains.} 
Our screening rule adds several computational overheads to the original online optimization problem, and all of them are of $\Oo(n)$ complexity. 
Denote by $n_t$ the dimension of the problem at current iteration.
\begin{itemize}
\item For the {\em inner loop} of Algorithm \ref{algo:screen}, \texttt{line 10-12} computing the dual certificate and primal/dual function values are of $\Oo(n_t)$ complexity. 
\item For the {\em outer loop} of Algorithm \ref{algo:screen}, all computations are at most $\Oo(n_t)$. 
\end{itemize}
Overall, the computational overheads added by screening is $\Oo(n_t)$ per iteration where $n_t$ is the dimension of $\beta_t$ at iteration step $t$. 

{\hspace{12pt}}On the other hand, our screening rule can effectively remove useless features along iteration. 
Suppose the sparsity of $\beta^\star$ is $\kappa$ which is much smaller than $n$ and our screening rule manages to screen out all useless features, then eventually $n_t = \kappa$ for all $t$ large enough, which in turn means the computational overheads are negligible. 

\item[{\rm (ii)}] {\bf Effect of the exponent $w$.} 
For Algorithm \ref{algo:screen}, the weight parameter $\mu_t$, specified by the exponent $w$, determines how important the latest iterate is. As a result, $w$ is crucial to the screening behaviour of Algorithm \ref{algo:screen}. In general the value of $w$ lies in $]0.5, 1]$. As we shall see in the numerical experiments, the smaller the value of $w$, the more aggressive the screening rule which makes Algorithm \ref{algo:screen} unsafe. While for larger choice of $w$, the screening is much more passive, hence safer.

\item[{\rm (iii)}] {\bf Choices of $T$.} 
For Algorithm \ref{algo:screen}, the inner loop iteration number is controlled by $T$. 
Similar to Algorithm \ref{algo:full-screen}, in practice, choices like $\ell m$ with $\ell$ being small integers demonstrate good performance. 
\end{itemize}

\end{remark}

\subsection{Safety checks}\label{sec:section_checks}
Though our screening rule is adapted from gap-safe rule, which is guaranteed to be safe, \ie only removes useless features and keeps all the active ones, applying Algorithm \ref{algo:screen} {\it alone} is {\em not} guaranteed to be safe. 
This is due to the fact that the rule we derive is with respect to the online objective $\bar P^{(t)}$ which is not the original objective \eqref{eq:onlineprob}. 
As a result, potentially our screening rule can falsely remove useful features.
However, this can be avoided by incorporating safe guard step, for instance, we can combine Algorithm \ref{algo:screen} with the strong rules developed in~\cite{tibshirani2012strong} to avoid false removal.

In the online setting, it is impossible to check the optimality condition as in the strong rules paper of~\cite{tibshirani2012strong} to avoid false removal. However we can offer confidence intervals on the safety of the reconstructed solution: 
\begin{itemize}
\item [{\rm (i)}]
Given a computed solution $\beta$ and support $S$, we can check the optimality of $\beta$ by computing for $(x_s,y_s)\overset{iid}{\sim} \Lambda$, $s=1,\ldots, K$ for $K>0$ and
$
\hat Z \eqdef \sfrac{1}{K} \msum_{s=1}^K f_{y_s}'(x_s^\top \beta) x_s.
$
Note that if $Z\eqdef \EE[\hat Z]$ satisfies $\norm{Z}_\infty \leq 1$ and $Z_S = \sign(\beta_S)$, then $\beta$ is indeed an optimal solution. 

\item [{\rm (ii)}]
By Hoeffding's inequality,
$
\PP\pa{\norm{\hat Z - \EE[\hat Z]}_\infty \geq \epsilon} \leq 2n \exp\bPa{-\sfrac{2K\epsilon^2}{G}}
$
where $G \geq \norm{f_y'(x^\top\beta) x}_\infty$ for all $x,y\in\Xx\times\Yy$. So, $\norm{Z_{S^c}}_\infty <1$ with probability at least $1-K^{-\alpha}$, provided that $
\norm{\hat Z_{S^c}}_\infty 
< 1 - \alpha \sqrt{\sfrac{G}{2K}}\log(2nK).
$
\end{itemize}
In implementation, one can periodically compute $\hat Z$ to check the safety of the computed support with confidence estimates.  

\begin{remark}
It can also be noted that in both Algorithm \ref{algo:full-screen} and Algorithm \ref{algo:screen}, the screening will be carried out until the termination of the iteration which actually is not necessary. Therefore in practice, one can terminate the screening once the support of the iterates is small enough. Take online screening for example, one terminate the screening if the size of the support of $\beta_t$ drops below $n/100$. However, the safety check using $\hat Z$ should be continue until the termination of the algorithm to ensure safeness.
\end{remark}

\section{Numerical results }\label{sec:experiments}

In this part, we present experiments to demonstrate the performance of our proposed online screen algorithm\footnote{Matlab code for reproducing our experiments are available at \url{https://github.com/jliang993/sgd-screening}}. All the experiments are performed on a ThinkStation P620 with 32-core CPU, 256GB memory and Ubuntu 20.04 system.

\subsection{Online experiments}

We first consider an online problem of the following form,
\[
\min_{\beta \in \RR^n} \sfrac12 \EE_{(x,y)}[\norm{\dotp{x}{\beta} - y}_2^2] + \lambda \norm{\beta}_1
\]
where $x$ is drawn from the uniform distribution  on $[-1,1]^n$ with $n = 10^5$ and $y = \dotp{\beta^\star}{x}+\varepsilon$ for a sparse vector $\beta^\star \in \bbR^n$ with $9$ non-zero entries and $\varepsilon\sim\calN(0,1)$ is Gaussian noise with mean 0 and variance $1$. Both standard Prox-SGD and Prox-SGD with online screening (``OS-Prox-SGD'') are considered. 
For OS-Prox-SGD, the iteration contains two phases: 
\begin{itemize}
\item[(a)] For the first phase, which accounts for $50\%$ of total number of iteration,  screening is  not applied and it is simply the plain Prox-SGD. 
\item[(b)] For the second phase, online screen is applied. In this part, since screening can reduce dimension of $\beta_t$, we can consider two scenarios\footnote{Note that this experiment is synthetic, and we show both cases simply to illustrate what happens if dimension reduction can also be exploited for more efficient sampling}: either draw new samples $x$ in the smaller dimension or in the original dimension $n$. As a result, two different implementations are considered
\begin{itemize}
	\item[(1)] We do not reduce the sample complexity, \ie sample $x_i$ in the original space $\mathbb{R}^n$. Denoted as ``OS-Prox-SGD-1''.
	\item[(2)] Sample in the smaller dimension obtained by screening, i.e. we  draw $x_i$ from  $\mathbb{R}^{|\supp(\beta_t)|}$ hence reducing sample complexity. Denoted as ``OS-Prox-SGD-2''.
\end{itemize}
\end{itemize}

\begin{figure}[!ht]
	\centering
	\subfloat[Size of $\beta_t$ over time]{ \includegraphics[width=0.45\linewidth]{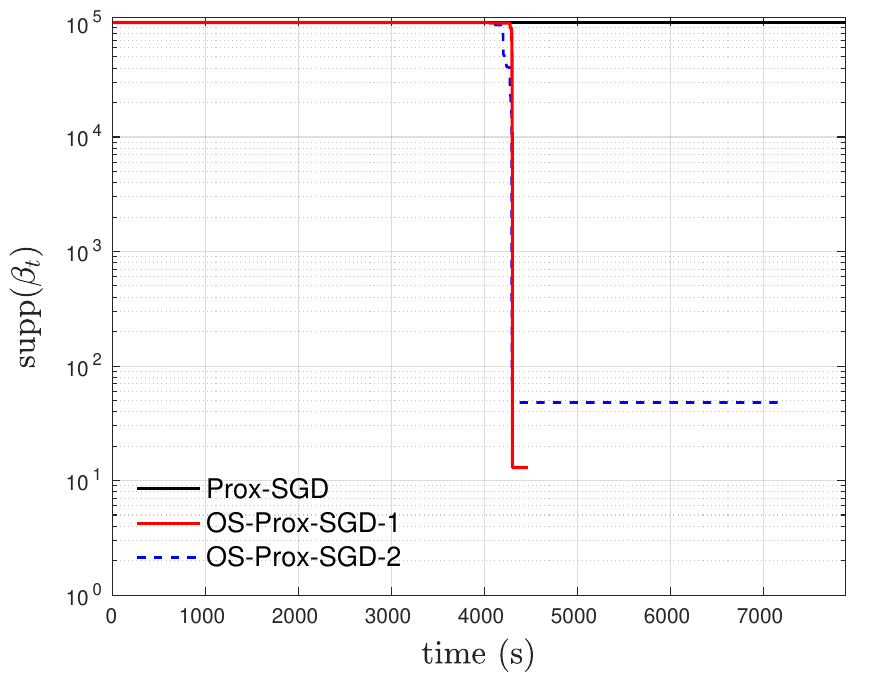} }		 
	\subfloat[Relative error $\norm{\beta_t-\beta_{t-1}}$]{ \includegraphics[width=0.45\linewidth]{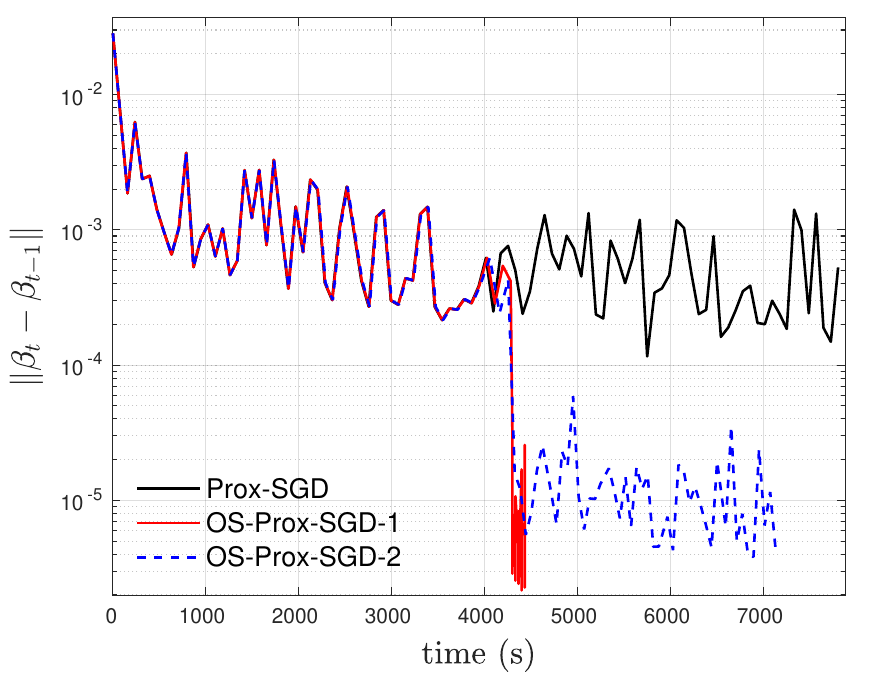} } 		\\[-2.5mm]
	\subfloat[Solutions comparisons]{ \includegraphics[width=0.45\linewidth]{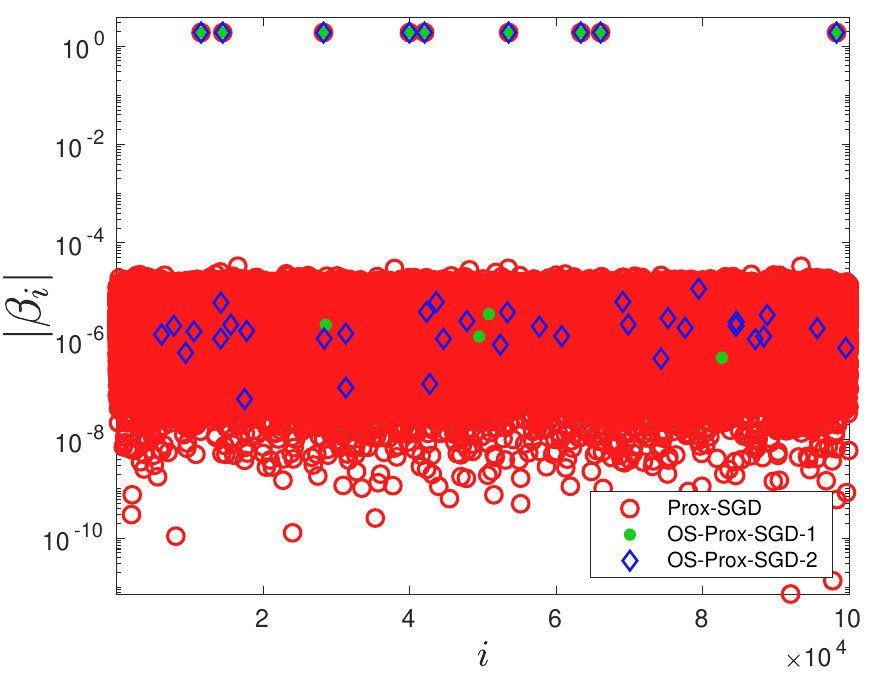} }		 
	\subfloat[Size of support vs CPU time]{ \includegraphics[width=0.45\linewidth]{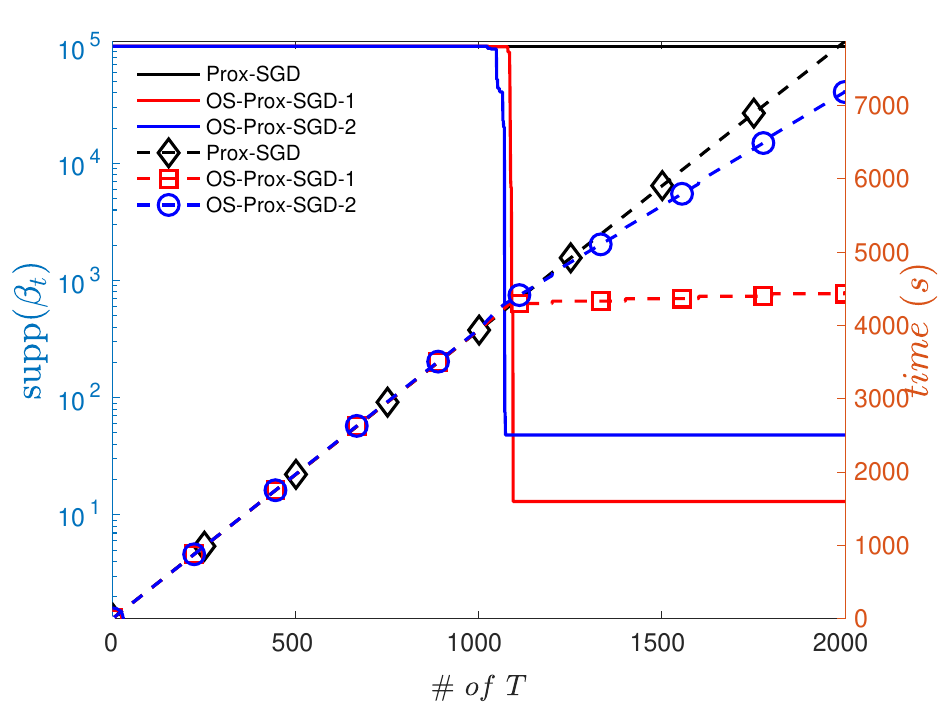} } 		\\[-3mm]
	\caption{Comparison between plain Prox-SGD and OS-Prox-SGD for a synthetic example. \rev{In (d), the dashed lines show time in seconds against iterations and the solid lines show the support against iterations.}  }
	\label{fig:online2}
\end{figure}

{\noindent}The maximum number of iteration in this test is set as $10^7$, and obtained observations are provided in Figure \ref{fig:online2}: 
\begin{itemize}
	\item Figure \ref{fig:online2}(a) shows the dimension of $(\beta_t)_{t}$ over time. As can be seen, our screening scheme manages to significantly reduce the dimension of the variable. {Note that for the presented example, OS-Prox-SGD-1 provides better dimension reduction that OS-Prox-SGD-2. In our implementations, we also observed cases where OS-Prox-SGD-2 provides better dimension reduction. This very likely is caused by the sampling step since the samples corresponding to the non-zero elements of $\beta_t$ are different.}

	\item In Figure \ref{fig:online2}(b), relative error also becomes smaller after screening starts, this is mainly because almost all very small elements (around scale $10^{-5}$) are screened out. 
	\item Figure \ref{fig:online2}(c) demonstrates the outputs of the two schemes, from which we observe that online screening effectively reduces the dimension of the problem. 
	
	\item Lastly in Figure \ref{fig:online2}(d), we provide a comparison between size of support of $\beta_t$ and wall clock CPU time. For the horizontal axis, we set $T = 5\times 10^3$. 
	We have the following wall clock CPU time for the three schemes. 
	\begin{table}[H]
	\begin{center}
		\begin{tabular}{c|c|c|c}
			\specialrule{.125em}{0ex}{0em} 
			Method & Prox-SGD & OS-Prox-SGD-1 & OS-Prox-SGD-2 \\ \specialrule{.125em}{0ex}{0em} 
			overall CPU time (s) & 8852 & 5044 & 7828 \\ \hline
			sampling time (s) & 5973 & 3273 & 5886 \\ \hline
			residual (s) & 2879 & 1771 & 1942 \\  \specialrule{.125em}{0ex}{0em} 
		\end{tabular}
	\end{center}
\end{table}
\vspace{-12mm}
It can be seen that, in terms of computational time, online screening  provides around 30\% or even more acceleration. While in terms of sample time, nearly 50\% the sampling time of Prox-SGD is saved. 

\end{itemize}
In terms of real-world data, we also consider the extended MNIST dataset, MNIST8m\footnote{\url{https://www.csie.ntu.edu.tw/~cjlin/libsvm/}}, which contains more than 8 million images of digits. Digits $4$ and $9$ are used for the  experiments, in total there are more than 1.5 million images of them. Iteration with only one pass through the data is made, hence, this can be treated as an online problem. 
In Figure \ref{fig:online3} we provide our numerical observation, which is very close to the observations in Figure \ref{fig:online2}, except in this experiment, we do not observe improvements in  running time. 
This is mainly caused by two factors: the small dimension of the problem and the images are sparse which makes the coefficients of Prox-SGD is sparse.  Nonetheless, it demonstrates the ability of screening to precisely identify relevant features.

\begin{figure}[!ht]
	\centering
	\subfloat[Solutions comparisons]{ \includegraphics[width=0.45\linewidth]{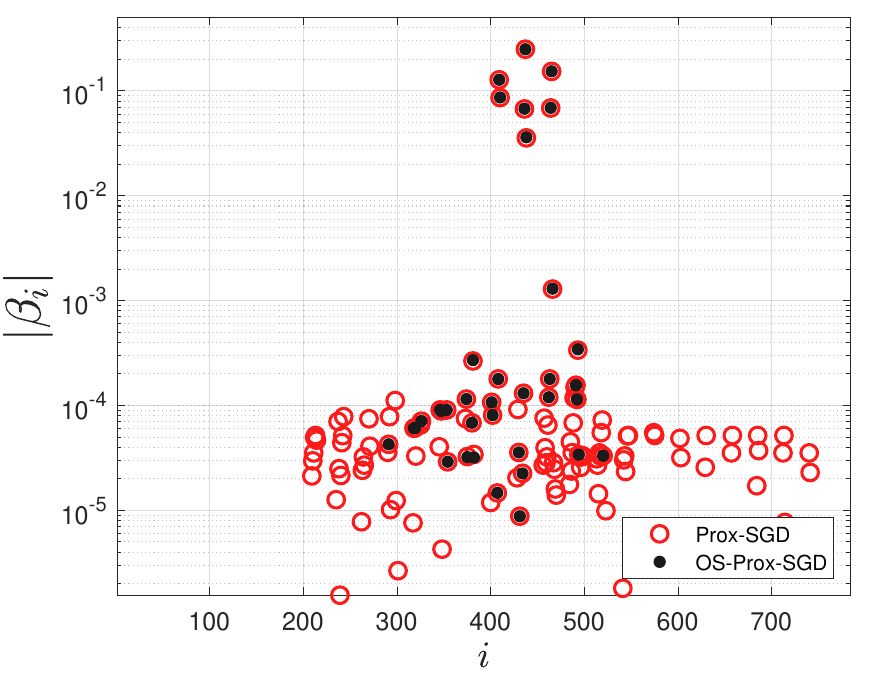} }		 
	\subfloat[Size of support vs CPU time]{ \includegraphics[width=0.45\linewidth]{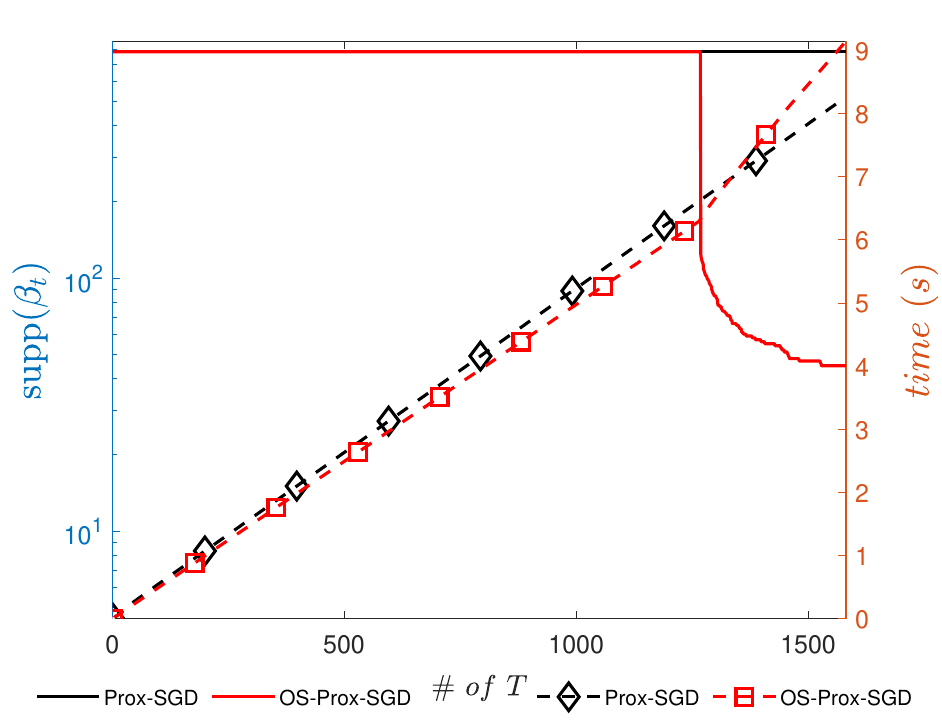} } 		\\ [-3mm]	
	\caption{Comparison between plain Prox-SGD and OS-Prox-SGD on MNIST8m dataset.} 
	\label{fig:online3}
\end{figure}

\subsection{More experiments on LIBSVM data}

In this part, we present experiments for the following $\ell_1$-regularized finite sum problem
\[
\min_{\beta \in \RR^n} F(\beta)+ \rev{\lambda} \norm{\beta}_1
\]
where $F(\beta) \eqdef \frac1m \sum_{i=1}^m f(x_i^\top \beta;y_i) $ with $f$ being either
\begin{itemize}
\item[(i)] the quadratic loss
$f(z;y) = \frac12(z-y)^2 $, a.k.a. the LASSO formulation,
\item[(ii)] the logistic loss $f(z;y) = \log\pa{1+\exp(-yz)}$ with $y\in \ens{-1,1}$, a.k.a. sparse logistic regression (SLR).
\end{itemize}
For both cases, we compare the performances of the standard proximal stochastic gradient descent (Prox-SGD), Algorithm \ref{algo:full-screen} (FS-Prox-SGD), Algorithm \ref{algo:screen} (OS-Prox-SGD). \rev{Both screening operations will be terminated if the size of the support $\beta_t$ drops below $20$.  For OS-Prox-SGD, the safety check is tested throughout the iterations. }
The details of settings of our experiments are as follows: 
\begin{itemize}
\item The SAGA algorithm~\cite{saga14} is used for computing the global minimizer of the problems. The exponent $w$ in Algorithm \ref{algo:screen} {\tt line 1} is set as $0.51$ for all tests.

\item 
Step-sizes of three algorithms are the same, which is $ \gamma_t = \frac{1}{mL t^{0.51}} $.

\item The maximum number of iterations for all schemes is set as the $\max(3\times 10^6, 300m)$. Both FS-Prox-SGD and OS-Prox-SGD  will be terminated if either maximum number of iteration is reached or the wall-clock CPU time exceeds that of Prox-SGD. 

\item Regularization parameter: we choose $\lambda < \lambda_{\max} \eqdef \norm{\nabla F(0)}_\infty$.  
In our experiments for the LASSO problem, we choose $\lambda = \frac{\lambda_{\textrm{max}}}{2}$. While for SLR problem, various choices are chosen and provided below.

\item For both screening schemes, we set $T = 4m$, \ie screening is applied every $4m$ steps. 

\item \rev{{\bf Safety checks} \; Every $5\times 10^5$ steps, we apply a safety check --- for the current iterate $\beta_t$, we compute the full certificate, find the coefficients where the optimality condition is violated and then add them back to the support of the $\beta_t$. If the maximum number of iteration is $3\times 10^6$, then in total we will have 6 times safety check.}

\end{itemize}
The following LIBSVM datasets are used, with four relatively small-scale ones and four large-scale ones. \footnote{All datasets can be downloaded from \url{https://archive.ics.uci.edu/ml/datasets.php} and \url{https://www.csie.ntu.edu.tw/~cjlin/libsvm/}.} 

\renewcommand{\arraystretch}{1.1}
\begin{table}[H]
	\begin{center}
	
		\begin{tabular}{c|c|c}
			\specialrule{.125em}{0ex}{0em} 
			Name & $m$ & $n$ \\ \specialrule{.125em}{0ex}{0em} 
			\texttt{colon-cancer} & 62 & 2,000 \\ \hline
			\texttt{leukemia} & 38 & 7,129 \\ \hline
			\texttt{breast-cancer} & 44 & 7,129 \\ \hline
			\texttt{gisette} & 6,000 & 5,000 \\ \specialrule{.125em}{0ex}{0em} 

		\end{tabular}
	\hskip3mm
	\begin{tabular}{c|c|c}
		\specialrule{.125em}{0ex}{0em} 
		Name & $m$ & $n$ \\ \specialrule{.125em}{0ex}{0em} 

		\texttt{arcene} & 200 & 10,000 \\ \hline
		\texttt{dexter} & 600 & 20,000 \\ \hline 
		\texttt{dorothea} & 1,150 & 100,000 \\ \hline
		\texttt{rcv1} & 20,242 & 47,236 \\ \specialrule{.125em}{0ex}{0em}
	\end{tabular}
		\caption{The considered datasets and their scales: $m$ is the number of samples and $n$ is the dimension of the problem. 	\label{tab:data}		}
	\end{center}
\end{table}

\subsubsection{Dimension reduction of screening schemes}

We first compare the support identification properties of Prox-SGD, FS-Prox-SGD and OS-Prox-SGD,
which are shown in Figure \ref{fig:lasso_support} (LASSO) and Figure \ref{fig:slr_support} (SLR), respectively. 
For each figure, two quantities are provide: size of support over number of epochs of $\beta_t$ for {\it solid lines} and elapsed time over number of epochs for {\it dashed lines}. 
For LASSO problem, we obtain the following observations,
\begin{itemize}
\item Prox-SGD, black lines in all figures, indeed does not have support identification property, as the size of support is oscillating and does not decrease.

\item \rev{Both FS-Prox-SGD and OS-Prox-SGD can effectively reduce the dimension of the iterates and provide CPU time-gain. The dimension reduction of OS-Prox-SGD {in general} is sharper than that of FS-Prox-SGD, which means it can significantly reduce the dimension of the problem at the very early stages.}
\end{itemize}

\begin{figure}[!ht]
	\centering
	\includegraphics[width=0.975\linewidth]{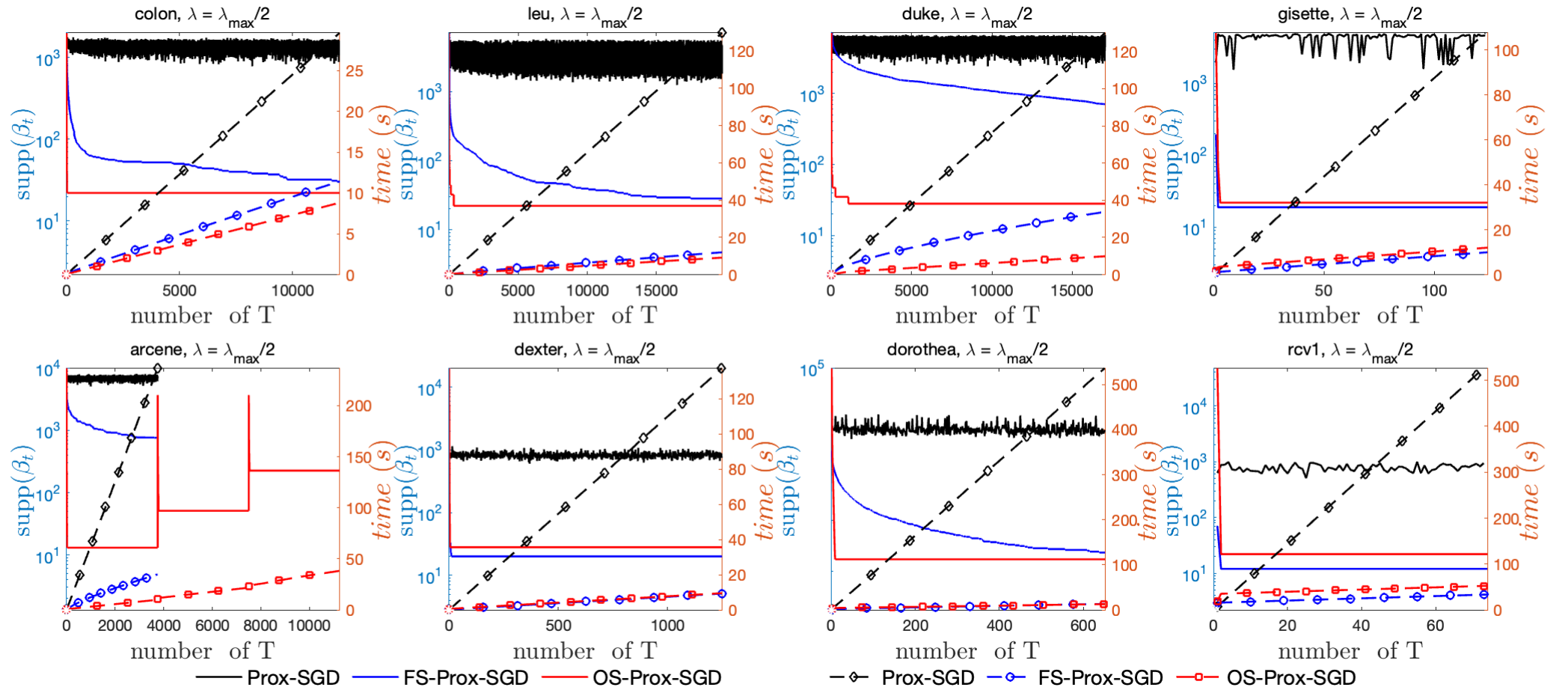} 	\\ [-3mm]		
	\caption{Dimension reduction vs wall clock time for LASSO problems. For all datasets, we set $\lambda = \frac{\lambda_{\textrm{max}}}{2}$. For each figure, two quantities are shown: {\it solid lines} for the dimension of the problem over iteration, {\it dashed lines} for CPU time over number of $T$. 
	} 
	\label{fig:lasso_support}
\end{figure}

\begin{remark}
It can be observed that for the {\tt arcene} dataset, the support of online screening has two jumps which is caused by our safety check operation. Once our safety check identifies unsafeness of the current iteration, we reset the iteration to the last safe state with extra random perturbation which results in the jump of support size.  When reset happens, we also increase the value of $w$ so as to make our online screening less aggressive. We shall discuss this in more details in Section \ref{sec:safety_check}. 
\end{remark}

For SLR problems, the choice of $\lambda$ is provided in each sub-figure of Figure \ref{fig:slr_support}. 
\rev{Overall, similar observations are obtained compared to those of LASSO problems. For the {\tt arcene} data, no resets caused by the safety check.} 
We observe from above that, for both problems, when online-screening works, it can achieve dimension reduction at the very early stage of the iteration, which means practically it is more attractive than the full-screening scheme, since in practice, stochastic algorithms are run for limited number of epochs.

\begin{figure}[!ht]
	\centering	
	\includegraphics[width=0.975\linewidth]{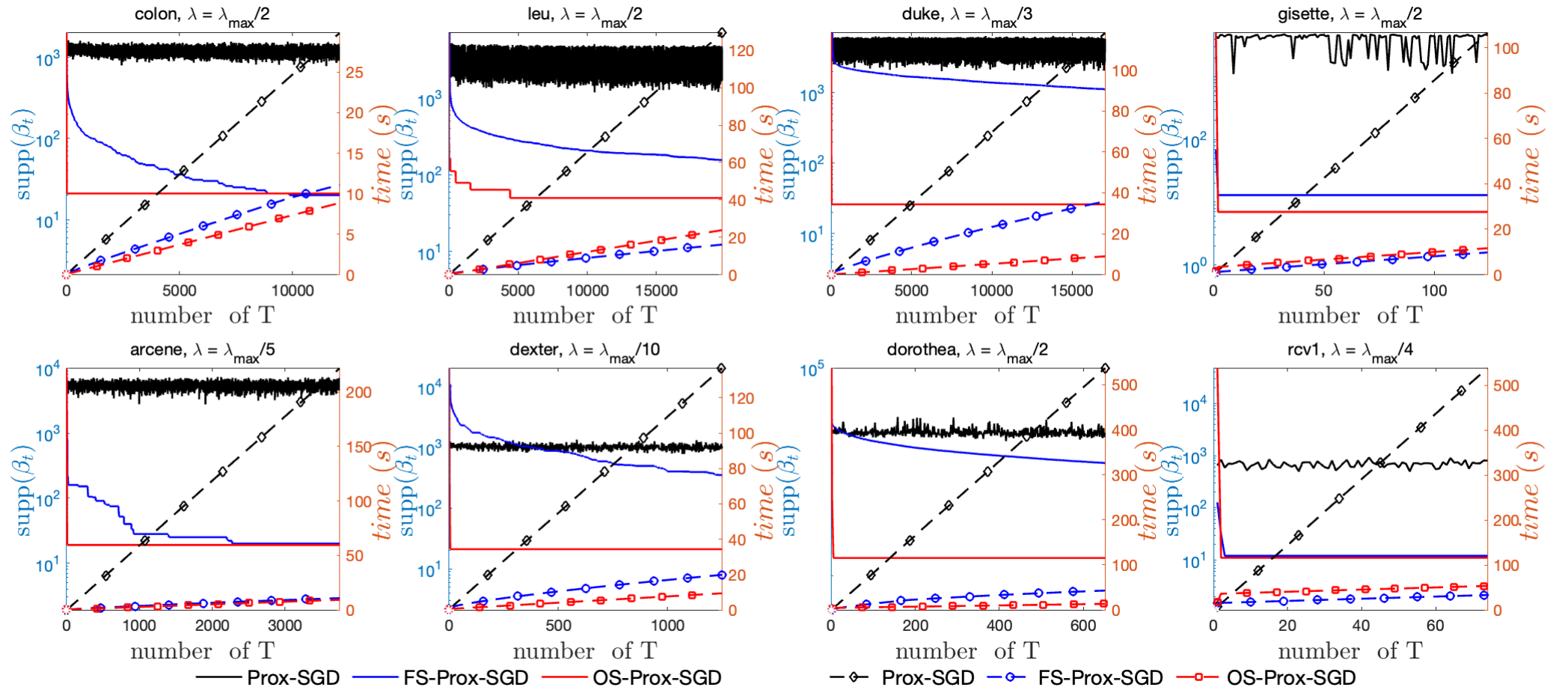} 	\\ [-3mm]		
	\caption{Comparison of support reduction and wall clock time for SLR problems. For each figure, two quantities are displayed: {\it solid lines} show the dimension of $\beta_t$ over number of epochs and {\it dashed lines} show the elapsed time over number of epochs. } 
	\label{fig:slr_support}
\end{figure}

\subsubsection{LASSO problem}

In this part, we present absolute error $\norm{\beta_t-\beta^\star}$ comparisons and solution quality comparisons for the LASSO problem. 
Error comparisons are displayed in Figure \ref{fig:lasso_error}. Similarly to the wall-clock time comparisons in Figure \ref{fig:lasso_support}, the faster algorithm yields faster error decays. \rev{For {\tt arcene} data, the resets of safety check also result in jump of error.}

\begin{figure}[!ht]
	\centering
	\includegraphics[width=0.975\linewidth]{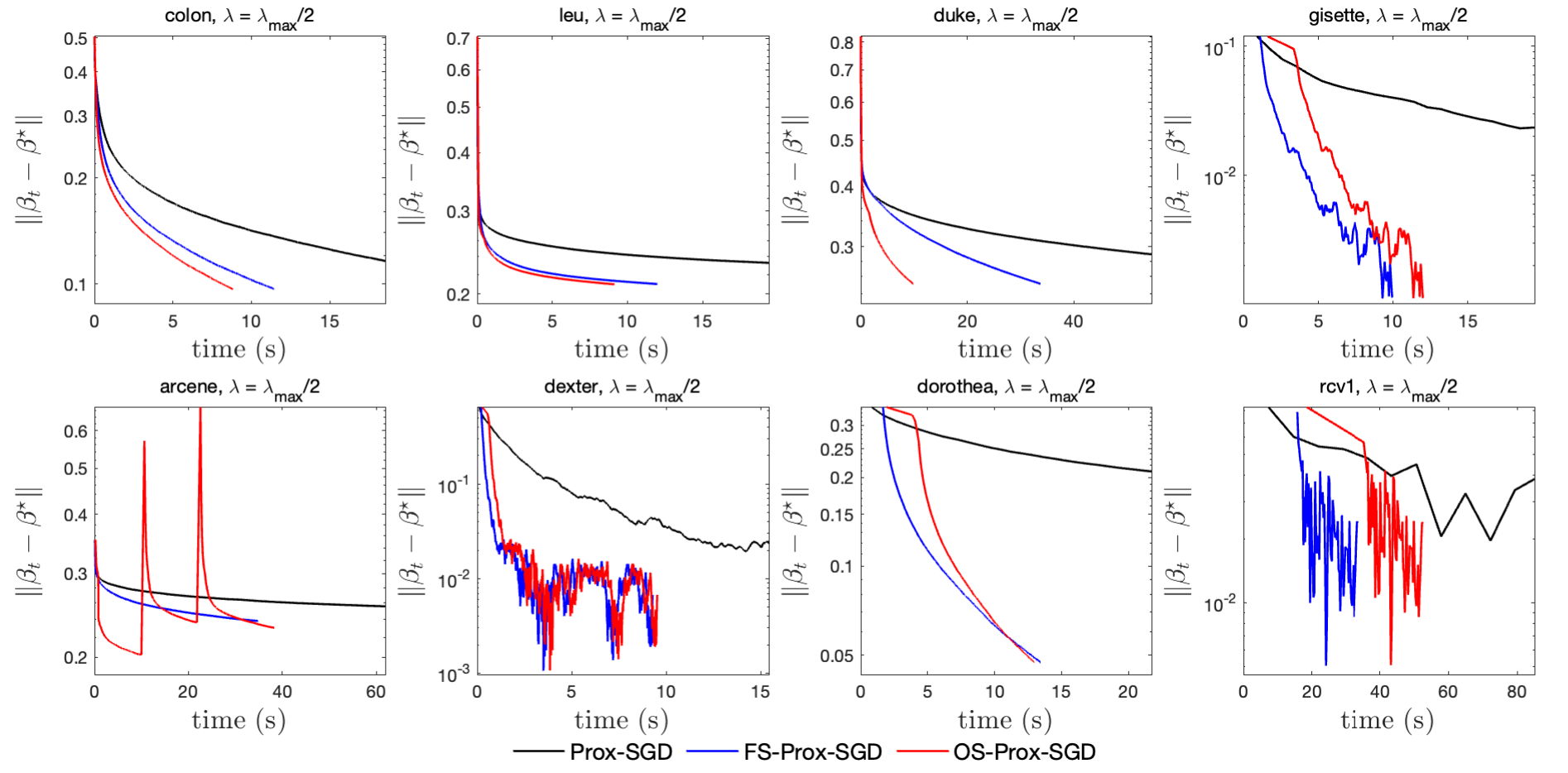} \\ [-3mm]	
	\caption{Comparison of errors $\norm{\beta_t-\beta^\star}$ against wall clock time for LASSO problems. For all datasets, the regularization parameter is $\lambda = \frac{\lambda_{\textrm{max}}}{2}$.} 
	\label{fig:lasso_error}
\end{figure}

In Figure \ref{fig:lasso_solution}, we provide comparisons of the final outputs obtained by the algorithms. 
For reference, the output of SAGA with optimality guarantee is included. 
\begin{itemize}
\item It can be observed that for Prox-SGD, non-identification can be observed by the large number of tiny values around or below $10^{-5}$. 

\item In general, there are discrepancies between the outputs of Prox-SGD schemes and the solution by SAGA, which means SGD schemes need more number of iterations. 

\item Screening can be effective in screening out these tiny values, with online-screening overall being slightly better than full-screening.
\end{itemize}
\begin{figure}[!ht]
	\centering
	\includegraphics[width=0.975\linewidth]{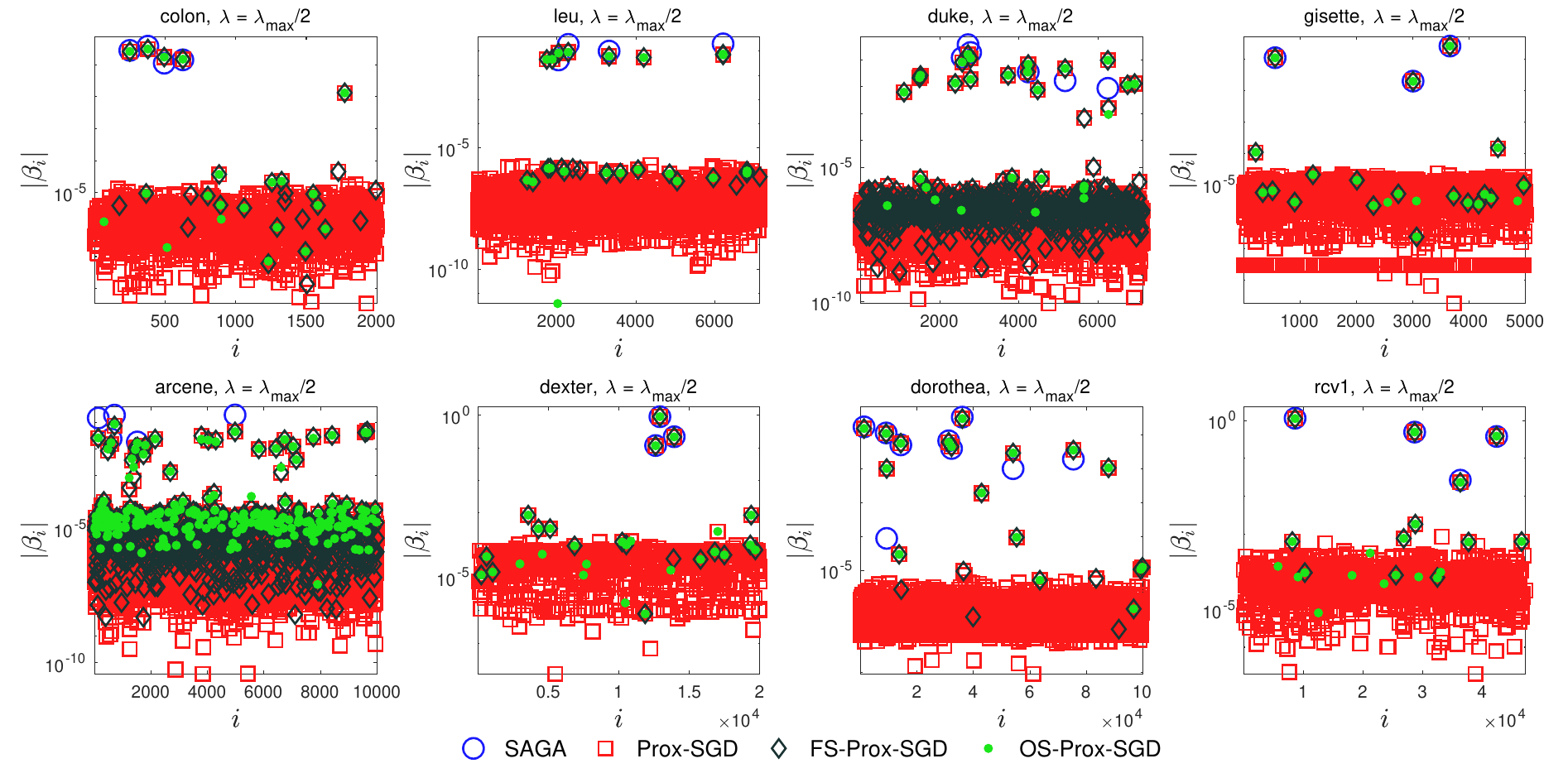} \\ 			
	\caption{Comparison of the solutions in terms of $\abs{\beta}$ for LASSO problems. We use the solution obtained by SAGA as `ground truth', and compare the final outputs of Prox-SGD, FS-Prox-SGD and OS-Prox-SGD. 
	} 
	\label{fig:lasso_solution}
\end{figure}

\subsubsection{Sparse logistic regression}

For SLR problems, the comparisons of error-time and solution quality are provided in Figure \ref{fig:slr_error} and Figure \ref{fig:slr_solution} respectively. 
Similar to the LASSO problem, the error comparison is in consistent with time comparisons of Figure \ref{fig:slr_support}. 

\begin{figure}[!ht]
	\centering
	\includegraphics[width=0.975\linewidth]{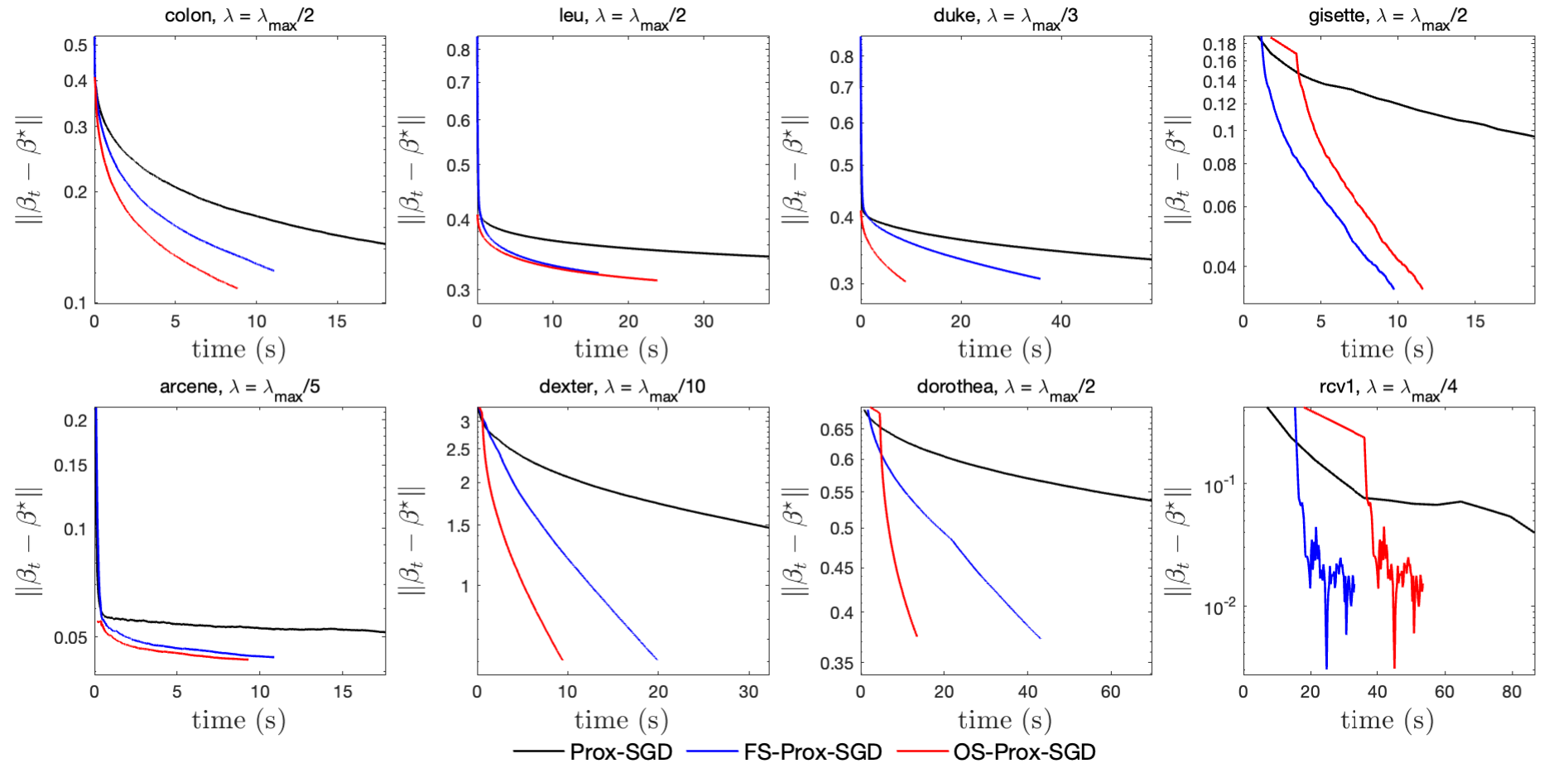}\\ [-3mm]		
	\caption{Comparison of error $\norm{\beta_t-\beta^\star}$ against wall clock time for SLR problems. } 
	\label{fig:slr_error}
\end{figure}

\begin{figure}[!ht]
	\centering
	\includegraphics[width=0.975\linewidth]{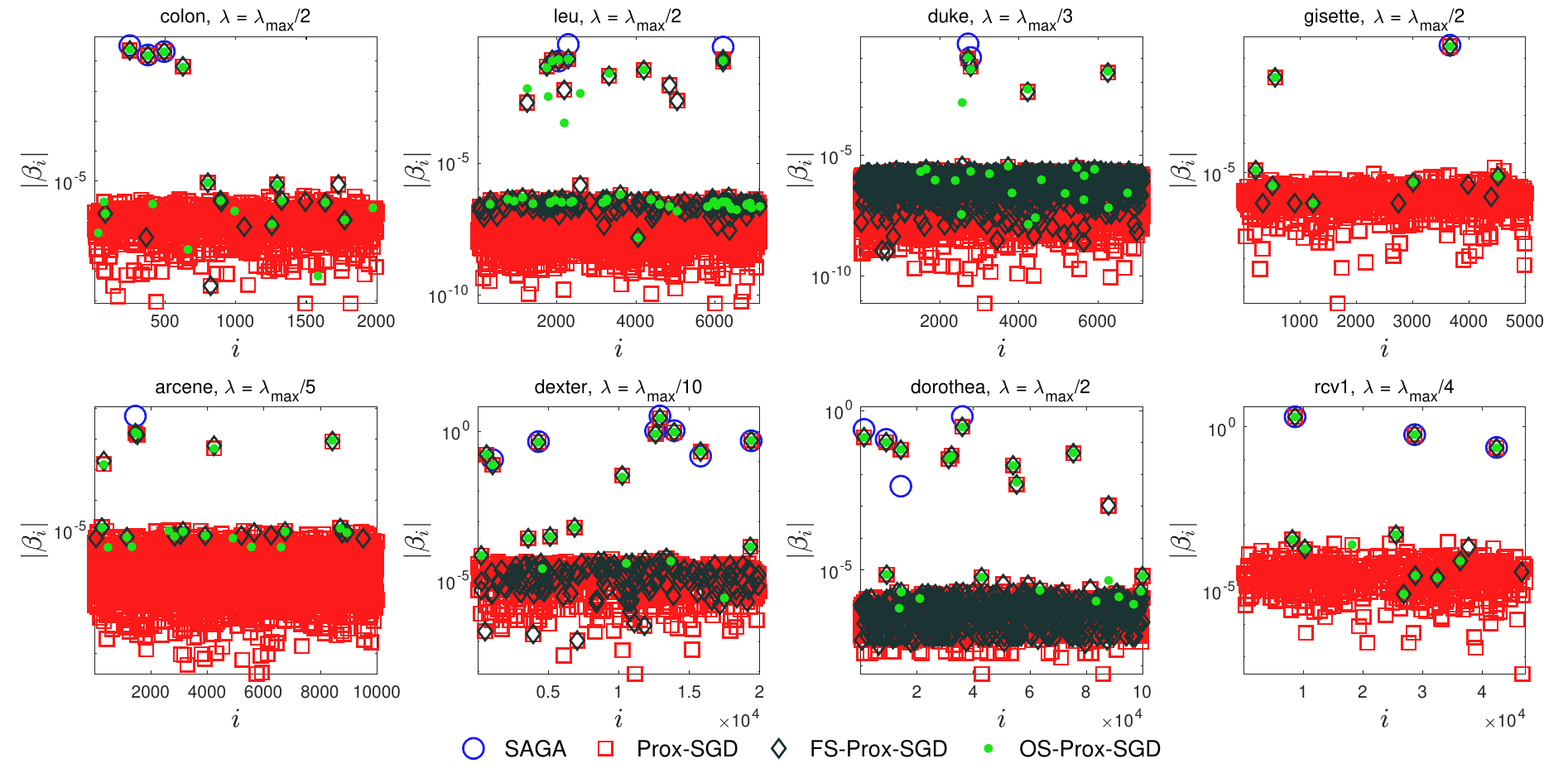} \\ [-3mm]	
	\caption{Comparison of the solutions in terms of $\abs{\beta}$ for SLR problems. We use the solution obtained by SAGA as `ground truth', and compare the final outputs of Prox-SGD, FS-Prox-SGD and OS-Prox-SGD. 
	}
	\label{fig:slr_solution}
\end{figure}

\subsubsection{Aggressiveness of $w$}\label{sec:safety_check}

In this last part of numeric experiments, we discuss the aggressiveness of the exponent parameter $w$ in Algorithm \ref{algo:screen} and comment on the resets caused by safety check of online-screening for LASSO on {\tt arcene} dataset. 
For the purpose of comparison, three initial choices of $w$ are tested, $w=0.51, 0.75$ and $0.99$. \rev{Note that $w$ should be smaller than $1$. } 
We describe the impact of $w$ over two quantities: dimension reduction and final output.

\rev{
In our implementation, the value of $w$ will increase by $0.1$ every time reset happens, until its value reaches $1$. 
The results for LASSO and {\tt arcene} dataset are provided below in Figure~\ref{fig:cmp_dexter}, from which we observe the followings
\begin{itemize}

\item {\bf Dimension reduction} For all choices of the exponent $w$, there is a sharp dimension reduction at the beginning stages of the iterations. The smaller the value of $w$, the sharper the reduction. 

\item {\bf Final output} The larger the value of $w$, the less sparse the output. Note that for initial value of $w=0.51$, after two resets, its values is increased to $0.71$. 
\end{itemize}
Finally, it is worth mentioning that the wall clock time for all three choices of $w$ are very close and around {10} seconds.
}
 
\begin{figure}[!ht]
	\centering
	\subfloat[Support size]{ \includegraphics[width=0.45\linewidth]{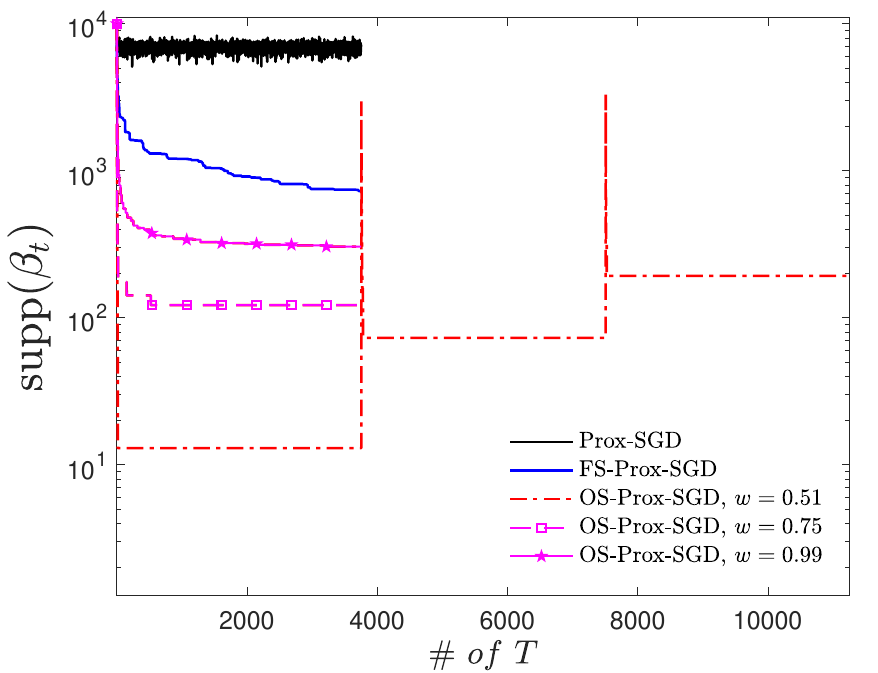} } 		
	\subfloat[Final output]{ \includegraphics[width=0.45\linewidth]{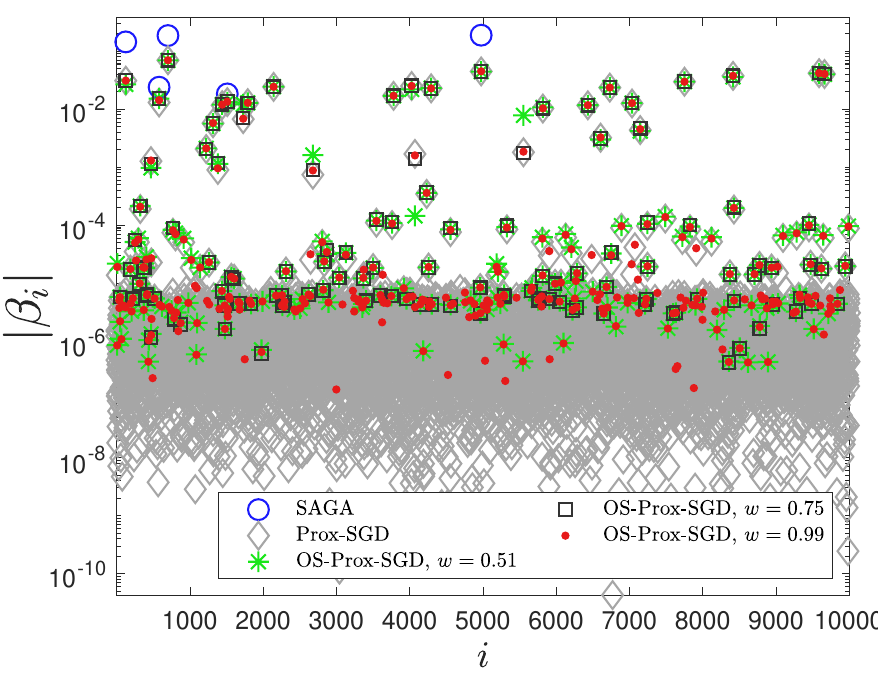} } \\ [-3mm]

	\caption{Comparison of the exponent parameter $w$ for LASSO problem and {\tt arcene} dataset. Figure (a) shows how the support decays over number of epochs and figure (b) the final outputs of the algorithms. } 
	\label{fig:cmp_dexter}
\end{figure}

\section{Conclusion}


Online optimization algorithms are widely used for solving large-scale problems arising from machine learning, data science and statistics. However, when combined with sparsity promoting regularizers, online methods can break the support identification property of these regularizers. In this paper, we combined the well established safe screening technique with online optimization methods which allows online methods to discard useless features along the iteration, hence achieving dimension reduction. Numerical result demonstrated that dramatic wall time gains can be achieved for classic regression tasks over real datasets.

\subsection*{Acknowledgements}
We would like to thank the anonymous reviewers who greatly help to improve the quality of this work. Jingwei Liang acknowledges support from the Shanghai Municipal Science and Technology Major Project (2021SHZDZX0102) and the support from SJTU and Huawei ExploreX Funding (SD6040004/033).

\vspace{\baselineskip}
\noindent The authors report there are no competing interests to declare.
\appendix


\section{Appendix}

\subsection{Preliminary}
\subsubsection{Convex analysis}

The sub-differential of a proper convex and lower semi-continuous function $\Omega: \RR^n \to \bbR \cup \ba{ +\infty}$ is a set-valued mapping defined by
\beq\label{eq:sub-diff}
\partial \Omega :  \RR^n \setvalued \RR^n ,~ \beta \mapsto \Ba{ Z\in\RR^n \,|\, \Omega( \beta') \geq \Omega( \beta) + \iprod{Z}{ \beta'- \beta} ,~~ \forall  \beta' \in \RR^n }.
\eeq

\begin{lemma}[Descent lemma \citep{bertsekas1999nonlinear}]\label{lem:descent}
	Suppose that $F: \RR^n \to \bbR$ is convex continuously differentiable and $\nabla F$ is $L$-Lipschitz continuous. Then, given any $ \beta, \beta'\in\RR^n$,
	\[
	F( \beta) \leq F(\beta') + \iprod{\nabla F(\beta')}{ \beta-\beta'} + \sfrac{L}{2}\norm{ \beta-\beta'}^2  .
	\]
\end{lemma}

 \begin{lemma}[Fenchel-Young inequality]
 Let $F:\RR^n\to \RR$ be proper, convex and lower semicontinuous, then for all $p,x\in\RR^n$,
$
 F(x) + F^*(p) \geq \dotp{p}{x}
 $
 with equality if $p\in \partial F(x)$.
 \end{lemma}

\rev{
\subsubsection{A stochastic Arzela-Ascoli result}
We recall a stochastic  Arzela-Ascoli result from  \cite{andrews1992generic}; see also  \cite[Theorem 11.3.2]{rao2008course}. Let $\Xx\subset \RR^d$ be compact, for each $x\in\Xx$ and $n\in\NN$, let $G_n(x) \; (= G_n(x)_\lambda)$ 
be a real measurable function on $\Lambda$. Let $B(x,\delta)$ be the closed ball of radius $\delta$ centred at $x$.

\begin{definition}
We say that $G_n(x)$ is strongly stochastically equi-continuous (SSE) on $\Xx$ if for all $\epsilon>0$ there exists $\delta>0$ such that
\[
\limsup_{n\to\infty} \PP\BPa{\sup_{x\in\Xx}\sup_{x'\in B(x,\delta)} \abs{G_n(x)-G_n(x')}>\epsilon} <\epsilon,
\]
and $\sup_{n\in\NN} \abs{G_n(x)}<\infty$ for all $x\in\Xx$ almost surely.
\end{definition}

{\noindent}To check that $G_n(x)$ is SSE, it is sufficient to show that
\begin{itemize}
\item[a)] $G_n(x) =  f_n(x)-\bar f_n(x)$ where $\bar f_n(x)$ is a non-random function that is continuous in $x$ uniformly over $x\in\Xx$ and $n\in\NN$.
\item[b)] For all $x$, $G_n(x)$ converge to 0 almost surely.
\item[c)] $\abs{f_n(x)-f_n(x')} \leq B_n\norm{x-x'}$ for all $x',x\in\Xx$ almost surely, where $B_n$ is a random variable such that $B_n = \Oo(1)$ almost surely.
\end{itemize}

\begin{theorem}[Stochastic Arzela-Ascoli]\label{thm:arzela}
Let  $\Xx$ be compact. Suppose that $G_n$ is SSE   and  point-wise convergent to 0 (for all $x$, $G_n(x) \to 0$ almost surely). Then almost surely
\[
\sup_{x\in\Xx} \abs{G_n(x)} \to 0, \qquad n\to\infty . 
\] 
\end{theorem}
}

\subsection{Derivation of the dual problem \eqref{eq:onlinedual}} \label{sec:onlinedual}

Denote $f_y(z) \eqdef f(z;y)$ and $v_\beta(x,y) = f_y'(x^\top \beta)$, applying  the Fenchel-Young (in)equality twice we obtain
\[
\begin{aligned}
\Pp(\beta) 
& =   \EE_{(x,y)}[v_\beta(x,y) x^\top \beta ] - \EE_{(x,y)}[f_y^*(v_\beta(x,y) )] + \lambda \Omega(\beta)\\
&=   - \EE_{(x,y)}[f_y^*(v_\beta(x,y) )] - \lambda\Pa{ \dotp{-\tfrac{1}{\lambda}\EE_{(x,y)}[v_\beta(x,y) x ]}{\beta} -  \Omega(\beta) }\\
&\geq - \EE_{(x,y)}[f_y^*(v_\beta(x,y) )] -\lambda \Omega^*\Pa{-\tfrac{1}{\lambda}\EE_{(x,y)}[v_\beta(x,y) x ]}
\end{aligned}\]
where the final line is an equality if  $-\frac{1}{\lambda}\EE_{(x,y)}[v_\beta(x,y) x] \in \partial \Omega(\beta)$, which is the case at an optimal primal solution $\beta^\star$. Therefore, it follows that
\[
\min_{\beta\in\RR^n}\Pp(\beta) = - \EE_{(x,y)}[f_y^*(v_{\beta^\star}(x,y) )] - \lambda \Omega^*\Pa{-\tfrac{1}{\lambda}\EE_{(x,y)}[v_{\beta^\star}(x,y) x ]} . 
\]
\rev{
On the other hand, taking any $\Lambda$-measurable function $v$,
\[
\begin{aligned}
\Dd(v)\eqdef &- \EE_{(x,y)}[f_y^*(v(x,y) )] -\lambda \Omega^*\Pa{-\tfrac{1}{\lambda}\EE_{(x,y)}[v(x,y) x ]}\\
 &\leq - \EE_{(x,y)}[f_y^*(v(x,y) )] - \lambda\Pa{ \dotp{-\tfrac{1}{\lambda}\EE_{(x,y)}[v(x,y) x ]}{\beta} -  \Omega(\beta) }
 \leq \EE_{(x,y)}[ f_y(x^\top \beta)] +\lambda\Omega(\beta)
\end{aligned}\]
where we apply the  Fenchel-Young inequality for the first inequality and the definition of convex conjugate for the second.

}
 Therefore, the dual problem is $\max_{v}  \Dd(v)$
and strong duality holds. Finally, note that $\Omega^*$ is the indicator function on the dual constraint set $\Kk_{\la,\eta}$.

\subsection{Proofs of Section \ref{sec:step1}}\label{proof:section3}

We prove Propositions  \ref{prop:obj_conv} and \ref{prop:dualconv} in this section. The proofs are provided for completeness although they use standard techniques, see for example \cite{lee2012manifold}. Similar results can be found in \cite{mairal2013stochastic} (though we relax the condition of $\sum_t \mu_t^2 \sqrt{t} < +\infty$ to simply $\sum_t \mu_t^2<+\infty$). We make use of the following lemma. 

\begin{lemma}[{Super-martingale convergence \citep{robbins1971convergence}}]\label{supermartingale}
Let $\Ff_k$ be a set of random variables with $\Ff_k \subset \Ff_{k+1}$ for all $k\in\NN$. Let $Y_k, Z_k, W_k$ be non-negative random variables which are functions of random variables in $\Ff_k$, such that
\begin{itemize}
\item[{\rm (i)}] $\EE[Y_{k+1}| \Ff_k] \leq Y_k + W_k - Z_k$. 
\item[{\rm (ii)}] $\sum_k W_k <+\infty$ with probability 1. 
\end{itemize}
Then, $\sum_k Z_k <+\infty$ and $Y_k$ converges to a non-negative random variable $Y$ with probability~1.
\end{lemma}

\begin{lemma}\label{lem:mairal}
For some $\ens{\mu_t}_t\subset (0,1)$ and random variables $\ens{f_t}_t \subset \RR$. 
Suppose $\sum_j \mu_j = +\infty$ and $\sum_j \mu_j^2 <+\infty$.
Let
 $\bar f_t\eqdef \mu_t f_t + (1-\mu_t) \bar f_{t-1}$ with $\bar f_1 = f_1$ and  $\eta_j^{(t)} \eqdef \mu_j \prod_{i=j+1}^{t} (1-\mu_{i})$. 
 Then the following results hold
 \begin{enumerate}
 \item[{\rm (i)}] $ \bar f_t = \sum_{j=1}^n \eta_j^{(t)} f_j $. 

 \item[{\rm (ii)}] $\sum_j \eta_j^{(t)}  = 1$.
 \item[{\rm (iii)}] $\lim_{t\to+\infty}\sum_{j=1}^t (\eta_j^{(t)} )^2 =  0$. \rev{In particular, 
 \[
 \msum_{j=1}^t (\eta_j^{(t)} )^2\lesssim \min_{m<t} { \msum_{j=m}^t \mu_j^2 + \exp\Pa{ -2\msum_{j=m}^t \mu_j } } \eqdef \epsilon_t.
 \]
  If $\mu_j = 1/j^\gamma$ with $\gamma\in (0.5,1]$, then $\epsilon_t = \Oo(t^{-\gamma+\frac12})$.}
 \item[{\rm (iv)}] Suppose that $\EE[f_j] = \EE[f_1]$ where $\ens{f_j}$ are iid random variables with $\abs{f_j} \leq B$. Then, $\lim_{n\to+\infty}\bar f_t = \EE[f_1]$ almost surely \rev{and
\[
\PP\Ppa{\abs{\bar f_t - \EE[f_1]} \geq v} \leq 2 \exp\BPa{-\sfrac{v^2}{8B^2 \sum_{j=1}^t(\eta_j^{(t)})^2  }} .
\]}
 \end{enumerate}
\end{lemma}
\begin{proof}
The first two statements are straightforward. For the third one, we have as $t\to+\infty$
\[
\begin{aligned}
\ssum_{j=1}^t (\eta_j^{(t)} )^2  
&= \ssum_{j=1}^t \mu_j^2 \mprod_{i=j+1}^{t} (1-\mu_{i})^2 
\leq \ssum_{j=m}^t \mu_j^2 + \mprod_{i=m}^{t} (1-\mu_{i})^2 \ssum_{j=1}^{m-1} \mu_j^2 . 
\end{aligned}
\]
Since by assumption, $\sum_j \mu_j^2 <+\infty$, we have  $\sum_{j=m}^t \mu_j^2 \to 0$  as $m,t\to+\infty$.
Moreover, $\lim_{t\to+\infty} \sum_{i=m}^t \mu_i = +\infty$. As a result, we have
\[
\begin{aligned}
 \mprod_{i=m}^{t} (1-\mu_{i})^2 \leq  \mprod_{i=m}^{t} \exp(-2\mu_{i}) = \exp\bPa{-2\msum_{i=m}^t \mu_i} \to 0 . 
\end{aligned}
\]
%
For $m\leq t$, define $Y_m \eqdef  \sum_{j=1}^m \eta_j^{(t)} (f_j   - \EE[f_1])$. Then, $\ens{Y_m}_{m\leq t}$ is a Martingale and $\abs{Y_m - Y_{m-1}} \leq 2\eta_m^{(t)} B$. By Azuma-Hoeffding inequality, given any $v>0$,
\[
\PP\Ppa{\abs{Y_t} \geq v} \leq 2 \exp\BPa{-\sfrac{v^2}{8B^2 \sum_{j=1}^t(\eta_j^{(t)})^2  }} .
\]
Hence $\bar f_t $ converges to $\EE[f_1]$ in probability.

 To show that it converges almost surely,  let $Y_t \eqdef \abs{(\bar f_t - \EE[f_1])}^2 $. We will make use of Lemma \ref{supermartingale} to show that this converges to 0 almost surely.   Note that
\[
\begin{aligned}
Y_t
&= \abs{(1-\mu_t)(\bar f_{t-1} - \EE[f_1] )+\mu_n  (f_t -  \EE[f_1] )}^2 \\
&= (1-\mu_n)^2 Y_{t-1} +\mu_n ^2  (f_t -  \EE[f_1] )^2 + (1-\mu_t) \mu_n (\bar f_{t-1} - \EE[f_1] )   (f_t -  \EE[f_1] ).
\end{aligned}\]
Taking expectation with respect to $\ens{f_j}_{j=1}^{t-1}$, it follows that
\[
\begin{aligned}
\EE_{t-1}[ Y_t ] &= (1-\mu_t)^2 Y_{n-1} +\mu_t ^2 \EE_{t-1}[  (f_t -  \EE[f_1] )^2 ] \leq Y_{t-1} + 4B^2 \mu_t^2.
\end{aligned}\]
Since $\sum_t \mu_t^2<+\infty$, it follows from Lemma \ref{supermartingale} that $Y_t$ converges almost surely, and this converges almost surely to 0 since $Y_t$ converges to 0 in probability. In particular, $\bar f_t$ converges to $\EE[f_1]$ almost surely. 
\end{proof}

Now we are ready to prove  Proposition \ref{prop:obj_conv}. 

\begin{proof}[Proof of Proposition \ref{prop:obj_conv}]
The proof makes use of Lemma \ref{lem:mairal}, which can be applied thanks to our assumption on $\mu_t$ in Remark \ref{eq:mu_assump}, that $\mu_t = t^{-u}$ for $u\in (0.5,1]$ such that $\sum_t \mu_t = +\infty$ and $\sum_t \mu_t^2 < + \infty$. 

\paragraph{Claim (i)} First note that for each $\beta$, with probability 1, $\abs{ \bar P^{(t)}(\beta) - P(\beta) } \to 0$ as $ t\to+\infty$, by (iv) of Lemma \ref{lem:mairal}. 
Note that $\bar P^{(t)}(\beta) - P(\beta)  = \bar F^{(t)}(\beta) - F(\beta) $ is \rev{strongly stochastically}  equi-continuous in $\beta$: by the mean value theorem there exists $\ens{\xi_s}_{s=1}^t \subset \Bb_R$, where $ \Bb_R$ is the ball of radius $R$ with $R \eqdef \sup_{x\in\Xx, \beta\in\Oo} \abs{\dotp{x}{\beta}}$, such that 
\[
\abs{\bar F^{(t)}(\beta)- \bar F^{(t)}(\beta')} = \abs{\ssum_{s=1}^t \eta_s^{(t)} f'_{y_s}(\xi_s) x_s^\top(\beta - \beta') } . 
\]
 For some  point $\xi_0\in \Bb_R$, 
\[
\begin{aligned}
\abs{\bar F^{(t)}(\beta)- \bar F^{(t)}(\beta')} &\leq  \abs{\ssum_{s=1}^t \eta_s^{(t)} (f'_{y_s}(\xi_s) -f'_{y_s}(\xi_0) )  x_s^\top(\beta - \beta') } 
+ \abs{\ssum_{s=1}^t \eta_s^{(t)} f'_{y_s}(\xi_0)  x_s^\top(\beta - \beta') } \\
&\leq  L \msum_{s=1}^t \eta_s^{(t)} \abs{\xi_s - \xi_0} \norm{x_s} \norm{\beta - \beta'} + \abs{\Pa{\ssum_{s=1}^t \eta_s^{(t)} f'_{y_s}(\xi_0)  x_s}^\top(\beta - \beta') } 
\end{aligned}\]
where we have used that fact that $f'_{y_s}$ is $L$-Lipschitz. By boundedness of $\Bb_R$ and $\Xx$, there exists $B$ such that
\[
 L \msum_{s=1}^t \eta_s^{(t)} \abs{\xi_s - \xi_0} \norm{x_s} \norm{\beta - \beta'}  \leq B   \norm{\beta - \beta'}  \msum_{s=1}^t \eta_s^{(t)}  = B  \norm{\beta - \beta'}.
 \]
By Lemma \ref{lem:mairal}~(iv), we know that with probability 1,
\[
\msum_{s=1}^t \eta_s^{(t)} f'_{y_s}(\xi_0)  x_s  \to \EE[f'_y(\xi_0) x] , \qquad  t\to+\infty.
\]
\rev{
Therefore, we arrive at
$
 \abs{\Pa{\ssum_{s=1}^t \eta_s^{(t)} f'_{y_s}(\xi_0)  x_s}^\top(\beta - \beta') }  \leq K_t \norm{\beta - \beta'}
$
 where $K_t \eqdef  \norm{\ssum_{s=1}^t \eta_s^{(t)} f'_{y_s}(\xi_0)  x_s}= \Oo(1)$. It follows that \[\abs{\bar F^{(t)}(\beta) - F(\beta)} \leq (B+K_t) \norm{\beta - \beta'}\] is stochastically equi-continuous in $\beta$ on compact sets.
} 
Hence, by Arzela-Ascoli (Theorem \ref{thm:arzela}), it follows that $\bar P^{(t)}$ converges uniformly to $P$ on compact sets.

%
\paragraph{Claim (ii)} 
Denote $Z^{(t)}(\beta) \eqdef -\frac{1}{\lambda}   \sum_{s=1}^t \eta_{s}^{(t)} f'_{y_s}(x_{s}^\top \beta ) x_{s}$.  Note that given $\beta^{(t)} \in\Argmin_\beta \bar P^{(t)}(\beta)$, $Z^{\star,(t)}= Z^{(t)}(\beta^{(t)})$.  We show convergence in the dual norm $\Omega^D$ (and convergence in infinity norm is then immediate from equivalence of norms in finite dimensions).
By the triangle inequality and recalling that $Z^* = \EE[-\tfrac{1}{\lambda} f'(x^\top \beta^\star, y ) x]$,
\begin{equation}\label{eq:conv_1}
\begin{aligned}
\Omega^D\Pa{Z^{(t)}(\beta^{(t)}) - Z^* } 
&\leq  \Omega^D\Pa{Z^{(t)}(\beta^{(t)}) -  Z^{(t)}(\beta^{\star}) } 
+  \Omega^D\Pa{Z^{(t)}(\beta^{\star}) - \EE[-\tfrac{1}{\lambda} f'(x^\top \beta^\star, y ) x] } .
\end{aligned}
\end{equation}
 To bound the first term on the RHS, let $\theta_s^{(t)} = f'(x_{s}^\top \beta^{(t)}, y_{s})$ and $\theta^\star_s =f'(x_{s}^\top \beta^\star, y_{s})$, then
\begin{align*}
\Omega^D\pa{Z^{(t)}(\beta^{(t)} + Z^{(t)}(\beta^*)} 
&= \sfrac{1}{\lambda} \Omega^D\big(  \ssum_{s=1}^t \eta_t^{(t)} \pa{ \theta_s^{(t)}  - \theta_s^*} x_s \big)\\
&\leq  \sfrac{1}{\lambda}   \sqrt{\ssum_{s=1}^t \eta_t^{(t)} \pa{ \theta_s^{(t)}  - \theta_s^*}^2}\sqrt{ \ssum_{s=1}^t\eta_s^{(t)} \Omega^D\left( x_s \right)^2}.
\end{align*} 
 By strong convexity of $\bar D^{(t)}$ (c.f. the proof  of Lemma \ref{lem:radius_gap}), we have
\[
\begin{aligned}
\sfrac{1}{2L} \msum_{s=1}^t \eta^{(t)}_s \abs{ \theta_{s}^{(t)} - \theta^{\star}_s}^2  &\leq \bar D^{(t)}(\theta^{(t)}) - \bar D^{(t)}(\theta^\star) \rev{+ \msum_{s=1}^t \eta_s^{(t)} f_{y_s}(0) \; \Pa{ \Omega^D(Z^{(t)}(\beta^{\star}))-1 }_+ }\\
&\leq \bar P^{(t)}(\beta^\star) - \bar D^{(t)}(\beta^\star) \rev{+ \msum_{s=1}^t \eta_s^{(t)} f_{y_s}(0) \;  \Omega^D(Z^{(t)}(\beta^{\star}) -\EE[Z^{(t)}(\beta^{\star})] )}.
\end{aligned}
\]
It follows that 
\begin{equation}\label{eq:interm}
\begin{aligned}
&\Omega^D\Pa{Z^{(t)}(\beta^{(t)}) -  Z^\star }  \\
&\leq \Omega^D\Pa{Z^{(t)}(\beta^{\star}) -\EE[Z^{(t)}(\beta^{\star})]} +\sfrac{1}{\la} \sqrt{ \ssum_{s=1}^t \eta^{(t)}_s \abs{ \theta_{s}^{(t)} - \theta^{*}_s}^2  } \sqrt{\ssum_{s=1}^t \eta_s^{(t)}\Omega^D(x_s)^2}\\
&\leq \Omega^D\Pa{Z^{(t)}(\beta^{\star}) -\EE[Z^{(t)}(\beta^{\star})]} +  \sfrac{2BL}{\la} \sqrt{ \bar P^{(t)}(\beta^\star) - \bar D^{(t)}(\beta^\star) \rev{+ G \; \Omega^D(Z^{(t)}(\beta^{\star}) -\EE[Z^{(t)}(\beta^{\star})] )} }
\end{aligned}
\end{equation}
where $G\geq \sup_{y\in \Yy} \abs{f_{y}(0)}$ and $B\geq \sup_{x\in\Xx} \Omega^D(x)$.

\paragraph{Almost sure convergence}
Note that   $\EE[\bar P^{(t)}(\beta^\star) - \bar D^{(t)}(\beta^\star) ] = 0$ by optimality of $\beta^\star$ and $\abs{ \bar P^{(t)}(\beta^\star) - \bar D^{(t)}(\beta^\star) }$ converges to zero in expectation and almost surely by Lemma \ref{lem:mairal} (iv).
So, by applying Lemma \ref{lem:mairal} (iv) to $Z^{(t)}(\beta^{\star}) -\EE[Z^{(t)}(\beta^{\star})]$ and  $\bar P^{(t)}(\beta^\star) - \bar D^{(t)}(\beta^\star)$ in  \eqref{eq:interm}, we have for $t\to+\infty$, \[\Omega^D\bPa{Z^{(t)}(\beta^{\star}) - Z^* } \to 0 . \] 

\paragraph{Convergence in expectation}
 Taking expectations in \eqref{eq:interm} yields
\[
\begin{aligned}
&\EE[\Omega^D\pa{Z^{(t)}(\beta^{(t)}) -  Z^\star} ]  
\leq \EE[\Omega^D(Z^{(t)}(\beta^{\star}) -\EE[Z^{(t)}(\beta^{\star})] )] \\
&\qquad\qquad + \sfrac{ 2BL}{\la} \sqrt{ \EE[\bar P^{(t)}(\beta^\star) - \bar D^{(t)}(\beta^\star) \rev{+ G \; \Omega^D(Z^{(t)}(\beta^{\star}) -\EE[Z^{(t)}(\beta^{\star})] )}] }\\
&= \EE[\Omega^D(Z^{(t)}(\beta^{\star}) -\EE[Z^{(t)}(\beta^{\star})] )] + \sfrac{ 2BL \sqrt{G}}{\la} \sqrt{ \EE[ \rev{ \Omega^D(Z^{(t)}(\beta^{\star}) -\EE[Z^{(t)}(\beta^{\star})] )}] }
\end{aligned}\]
where we applied Jensen's inequality for the inequality and used optimality of $\beta^{\star}$ to deduce $\EE[\bar P^{(t)}(\beta^\star) - \bar D^{(t)}(\beta^\star)] = 0$.
To bound the RHS, we note that (by the equivalence of norms) for some $C>0$,
\[
 \EE[\Omega^D(Z^{(t)}(\beta^{\star}) -\EE[Z^{(t)}(\beta^{\star})] )]  \leq  C \EE[\norm{Z^{(t)}(\beta^{\star}) -\EE[Z^{(t)}(\beta^{\star})] }_\infty].
 \] We then  apply  Lemma \ref{lem:mairal} to
 $f_s \eqdef \pa{ f_{y_s}'(x_s^\top \beta) x_s}_k$ for each $k=1,\ldots, n$ followed by the union bound to obtain
 \[
 \PP\pa{\norm{Z^{(t)}(\beta^{\star}) -\EE[Z^{(t)}(\beta^{\star})] }_\infty>v} \leq 2n \exp\bPa{-\sfrac{v^2}{8 G^2 B \sum_{s=1}^t (\eta_s^{(t)})^2}}.
 \]
 Therefore,
 \[
\begin{aligned}
  \EE \pa{\norm{Z^{(t)}(\beta^{\star}) -\EE[Z^{(t)}(\beta^{\star})] }_\infty>v}  &\leq 2n \int_0^\infty \exp\bPa{-\sfrac{v^2}{8 G^2 B \sum_{s=1}^t (\eta_s^{(t)})^2}} \mathrm{d}v
  \\
  &\lesssim 2n  \sqrt{ G^2 B \ssum_{s=1}^t (\eta_s^{(t)})^2
 } = \Oo(n \epsilon_t)
  \end{aligned}\]
  It follows $\EE[ \Omega^D(Z^{(t)}(\beta^{(t)}) - Z^\star)] = \Oo(\sqrt{\epsilon_t})$ where the implicit constant depends on $n$. \qedhere


%
\end{proof}

\begin{proof}[Proof of Proposition \ref{prop:dualconv}]

Let $\Ff_t$ be the $\sigma$-algebra generated by  $\ens{(x_s,y_s)}_{s\leq t}$. Note
\[
\begin{aligned}
&-\sfrac{1}{\lambda}\msum_{s=1}^t \eta_s^{(t)} \theta_s x_{s} - Z^\star \\
&=\sfrac{-1}{\lambda} \msum_{s=1}^t \eta_s^{(t)} \pa{ f_{y_{s}}'(x_{s}^\top \beta_s) x_{s} -   f_{y_{s}}'(x_{s}^\top \beta^{\star}) x_{s}   } +  \msum_{s=1}^t \eta_s^{(t)} \Pa{ \sfrac{-1}{\lambda} f_{y_{s}}'(x_{s}^\top \beta^{\star}) x_{s} - Z^\star } . 
\end{aligned}
\]
Therefore we get
\[
\begin{aligned}
\norm{-\sfrac{1}{\lambda}\ssum_{s=1}^t \eta_s^{(t)} \theta_s x_{s} - Z^\star }_\infty
&\leq \sfrac{B^2 L}{\lambda} \msum_{s=1}^t \eta_s^{(t)} \norm{\beta_s-\beta^{\star}}+ \norm{ \sfrac{1}{\lambda}\ssum_{s=1}^t \eta_s^{(t)} z_s  }_\infty ,
\end{aligned}
\]
where  $z_s \eqdef  f_{y_{s}}'(x_{s}^\top\beta^{\star}) x_{s} - \EE[ f_{y_{s}}'(x_{s}^\top\beta^{\star}) x_{s} ]$ and we used $\norm{x_s}\leq B$ and $f_y'$ is Lipschitz with constant $L$. 
\rev{ Taking expectations,
\[
\begin{aligned}
\EE \norm{-\sfrac{1}{\lambda}\ssum_{s=1}^t \eta_s^{(t)} \theta_s x_{s} - Z^\star }_\infty
&\leq \sfrac{B^2 L}{\lambda} \msum_{s=1}^t \eta_s^{(t)} \EE \norm{\beta_s-\beta^{\star}}+ \EE \norm{ \sfrac{1}{\lambda}\ssum_{s=1}^t \eta_s^{(t)} z_s  }_\infty. 
\end{aligned}\]}
By the Silverman-Toeplitz theorem \cite[Thm 1.1]{natarajan2017classical}, since  $\sum_s  \eta^{(t)}_s  = 1$ and $\norm{\beta_s-\beta^{\star}} \to 0$, $ \sum_{s=1}^t \eta_s^{(t)} \norm{\beta_s-\beta^{\star}} \to 0$ as $t\to\infty$. 
Fix $k\in [n]$, and define for $v\leq t$,
\[
Y_v^k \eqdef \msum_{s=1}^v \eta_s^{(t)}  (z_s)_k . 
\]
This is a martingale with bounded difference \[\abs{Y_v^k- Y_{v-1}^k} = \eta_v^{(t)} \abs{(z_v)_k} \leq 2 \abs{f_{y_v}'(x_v^\top \beta^{\star}) (x_v)_k} \eta_v^{(t)} \leq 2 B G \eta_v^{(t)} . \] By Azuma-Hoeffding inequality we get
$
\PP\Pa{ \abs{Y_t^k} \geq v} \leq 2\exp\Pa{ -\sfrac{v^2}{8B^2 G^2 \sum_{s=1}^t( \eta_s^{(t)})^2}  } 
$.
 Therefore, by the union bound,
$
\PP\Pa{\max_k \abs{Y_t^k} \geq v } \leq 2 n \exp\Pa{ -\sfrac{v^2}{8G^2 B^2 \sum_{s=1}^t( \eta_s^{(t)})^2}  }$ .
The RHS converges to 0 as $t\to +\infty$ by (iii) of Lemma \ref{lem:mairal} \rev{and
$
\EE[\norm{\ssum_{s=1}^t \eta_s^{(t)} z_s}_\infty] \lesssim n \sqrt{\ssum_{s=1}^t( \eta_s^{(t)})^2} 
$}
%
\rev{
and the following bound
\[
\begin{aligned}
\EE\norm{ -\tfrac{1}{\lambda} \ssum_{s=1}^t \eta_s^{(t)} \theta_s x_{s} - Z^\star  }_\infty & \lesssim  n \sqrt{\ssum_{s=1}^t( \eta_s^{(t)})^2} + \msum_s \eta_s^{(t)} \EE\norm{\beta_s - \beta^{\star}} . 
\end{aligned}\]
}
\vspace{-12mm}
\paragraph{Almost sure convergence} We apply Lemma \ref{supermartingale}. Let $Y_t = \abs{\bar Z_t - Z^\star}^2$, then,
\[
\begin{aligned}
\EE_{t-1}[Y_t ] &= (1-\mu_t) Y_{t-1} +  \mu_t^2 \EE_{t-1} \Pa{ -\sfrac{1}{\lambda}\theta_t x_t  - Z^\star }^2 
+(1-\mu_t) \mu_t Y_{t-1} \EE_{t-1} \Pa{ -\sfrac{1}{\lambda}\theta_t x_t  - Z^\star } .
\end{aligned}\]
Recall that $Z^\star = -\frac{1}{\lambda} \EE_{(x,y)}[f'(x^\top \beta^\star,y) x]$ and since $\norm{\beta_s} = \Oo(1)$ almost surely, letting $B_t \eqdef\EE_{t-1} \Pa{ -\sfrac{1}{\lambda}\theta_t x_t  - Z^\star }^2 =  \EE_{(x,y)}[(-\lambda^{-1}f_y'(x^\top \beta_t) x - Z^\star)^2]$ where we recall that $\theta_t = f'_{y_t}(x_t^\top \beta_t)$, $B_t = \Oo(1)$ almost surely.
\[
\begin{aligned}
\EE_{t-1}[Y_t ] &\leq (1-\mu_t)^2 Y_{t-1}  + \mu_t^2 B_t + (1-\mu_t) \mu_t Y_{t-1}  \sfrac{1}{\lambda} \EE_{(x,y)}[ \abs{(f'(x^\top \beta_t,y)- f'(x^\top \beta^\star,y)) x}]\\
&\leq (1-\mu_t)^2 Y_{t-1}  + \mu_t^2 B_t + \sfrac{LB^2}{\lambda} (1-\mu_t) \mu_t Y_{t-1} \norm{\beta_t - \beta^\star}\\
&=(1-\mu_t)(1-\mu_t + LB^2 \lambda^{-1} \mu_t\norm{\beta_t-\beta^\star}) Y_{t-1}  + \mu_t^2 B_t  . 
\end{aligned}\]
where we used the fact that $f_y'$ is $L$-Lipschitz and $\norm{x}\leq B$ for the second inequality.
Since $\beta_t\to \beta^\star$, for $t$ sufficiently large, $1-\mu_t + LB^2 \lambda^{-1} \mu_t\norm{\beta_t-\beta^\star} \in (0,1)$, so we can apply Lemma \ref{supermartingale} to conclude that $Y_t$ converges almost surely to 0.

\rev{
\paragraph{Convergence of regularization term}
Let us also establish the following convergence 
\[
-\msum_{s=1}^t f'_{y_s}(\dotp{x_s}{\beta_s})\dotp{x_s}{\beta_s} \to \lambda\Omega(\beta^{\star})  = -\EE f'_y(\dotp{x}{\beta^{\star}})\dotp{x}{\beta^{\star}} .
\]
Note that
\[
\begin{aligned}
&\abs{-\ssum_{s=1}^t \eta_s^{(t)} f'_{y_s}(\dotp{x_s}{\beta_s})\dotp{x_s}{\beta_s} -\lambda \Omega(\beta^{\star})} \\ 
&= \abs{ \ssum_{s=1}^t \eta_s^{(t)} f'_{y_s}(\dotp{x_s}{\beta_s})\dotp{x_s}{\beta_s}-\dotp{Z^\star}{\beta^{\star}}}\\
&\leq \abs{ \ssum_{s=1}^t \eta_s^{(t)} f'_{y_s}(\dotp{x_s}{\beta_s})\dotp{x_s}{\beta_s-\beta^{\star}}} + \abs{\ssum_{s=1}^t \eta_s^{(t)} f'_{y_s}(\dotp{x_s}{\beta_s})\dotp{x_s}{\beta^{\star}}-\dotp{Z^\star}{\beta^{\star}}}\\
&= \abs{ \ssum_{s=1}^t \eta_s^{(t)} f'_{y_s}(\dotp{x_s}{\beta_s})\dotp{x_s}{\beta_s-\beta^{\star}}} + \abs{\dotp{\bar Z_t}{\beta^{\star}}-\dotp{Z^\star}{\beta^{\star}}}\\
&\leq  B \msum_{s=1}^t \eta_s^{(t)} \abs{f_{y_s}'(x_s^\top \beta_s)} \norm{\beta_s-\beta^{\star}} + \norm{\bar Z_t - Z^\star} \norm{\beta^{\star}} 
\end{aligned}\]
which converges to 0 almost surely since almost surely we have $\beta_s \to \beta_*$ and 
\[
\abs{f_{y_s}'(x_s^\top \beta_s)}  
\leq \abs{f_{y_s}'(x_s^\top \beta_s) - f_{y_s}'(x_s^\top \beta^{\star})} + \abs{f_{y_s}'(x_s^\top \beta^{\star})} \leq L B \norm{\beta_s - \beta^{\star}} + G = \Oo(1) . 
\] 
Moreover, since $\EE[\dotp{\bar Z_t - Z^\star}{\beta^{\star}}] = 0$ and $\abs{f_{y_s}'(x_s^\top \beta^{\star})  x_s^\top \beta^{\star}} \leq G B\norm{\beta^{\star}}\eqdef B_0$, we arrive at
\[
\PP\Ppa{\abs{\dotp{\bar Z_t - Z^\star}{\beta^{\star}} } \geq v} \leq 2 \exp\BPa{-\sfrac{v^2}{8B_0^2 \sum_{j=1}^n(\eta_j^{(n)})^2  }} .
\]
\paragraph{Convergence of the duality gap}
To show convergence of $\gap_t(\bar \beta_t) \to 0$, recall that $\theta_s = f_{y_s}'(x_s^\top \beta_s)$ and by the Fenchel duality, 
\begin{equation}\label{eq:fenchel-dual-conseq}
\bar D^{(t)}((\theta_s)_{s\leq t}) 
= - \msum_{s=1}^t \eta_s f_{y_s}^*(\theta_s)
= \msum_{s=1}^t  \eta_s \bPa{ f_{y_s}(\dotp{x_s}{\beta_s})-  f_{y_s}'( \dotp{x_s}{\beta_s}) \dotp{x_s}{\beta_s} }.
\end{equation}
By adding and subtracting $ \bar P^{(t)}( \beta^{\star})  $ and using \eqref{eq:fenchel-dual-conseq},
\[
\begin{aligned}
\gap_t(\bar \beta_t)
&=
\bar P^{(t)}(\bar \beta_t) - \bar D^{(t)}((\theta_s)_{s\leq t}) \\
&=
\bar P^{(t)}(\bar \beta_t)  - \bar P^{(t)}( \beta^{\star}) + 
\msum_{s=1}^t  \eta_s  f_{y_s}(\dotp{x_s}{\beta^{\star}}) - \msum_{s=1}^t  \eta_s  f_{y_s}(\dotp{x_s}{\beta_s})\\
&\qquad +\lambda \Omega( \beta^{\star})  +  \msum_{s=1}^t  \eta_s   f_{y_s}'( \dotp{x_s}{\beta_s}) \dotp{x_s}{\beta_s} .
\end{aligned}\]
Note that, by the mean value theorem, for some $\xi_s$  between $x_s^\top \beta_s$ and $x_s^\top \beta^{\star}$,
\[
\begin{aligned}
&\abs{\ssum_{s=1}^t \eta_s f_{y_s}(x_s^\top\beta^{\star}) - f_{y_s}(x_s^\top \beta_s)}\\
&= \abs{\ssum_{s=1}^t \eta_s f_{y_s}'(\xi_s) (x_s^\top(\beta_s-\beta^{\star}))}\\
&=  \abs{\ssum_{s=1}^t \eta_s \pa{ f_{y_s}'(\xi_s) - f_{y_s}'(0) } (x_s^\top(\beta_s-\beta^{\star})) 
+ \ssum_{s=1}^t \eta_s  f_{y_s}'(0) (x_s^\top(\beta_s-\beta^{\star})) }\\
&\leq \msum_{s=1}^t \eta_s^{(t)} L B^2 \norm{\beta_s - \beta^{\star}}^2 + \msum_s G B \eta_s^{(t)}  \norm{\beta_s - \beta^{\star}} . 
\end{aligned}\]
The same argument can be applied to show
\[
\abs{\bar P^{(t)} (\bar \beta_t)  - \bar P^{(t)} ( \beta^{\star}) } 
\leq   L B^2 \norm{\bar \beta_t - \beta^{\star}}^2 +  G B\norm{\bar \beta_t - \beta^{\star}} . 
\]
As a result,
\[
\begin{aligned}
\abs{\gap_t(\bar \beta_t)}
&\leq C( \norm{\bar \beta_t - \beta^{\star}} )+  \msum_{s=1}^t \eta_s C(\norm{\beta_s - \beta^{\star}})+\abs{\lambda \Omega( \beta^{\star})  +  \msum_{s=1}^t  \eta_s   f_{y_s}'( \dotp{x_s}{\beta_s}) \dotp{x_s}{\beta_s} }
\end{aligned}\]
where $C(x) \eqdef  L B^2 x^2 +  G Bx $.
Convergence follows from (i).\qedhere
}

\end{proof}

\subsection{Proofs for Section \ref{sec:onlinescreen}}

\begin{proof}[Proof of Lemma \ref{lem:radius_gap}]
 Since $f_{y_{s}}$ is $L$-Lipschitz smooth, it follows that $ f_{y_{s}}^*$ is $\frac1L$-strongly convex, \rev{and for any $g_s^\star\in \partial f_{y_{s}}^*( \theta^{\star,(t)}_s)$}
\begin{equation}\label{eq:b1}
\begin{aligned}
\sfrac{1}{2L} \abs{ \theta_{s} - \theta^{\star,(t)}_s}^2 & \leq -f_{y_{s}}^*( \theta^{\star,(t)}_s) + f_{y_{s}}^*( \theta_{s}) - ( \theta_{s} - \theta^{\star,(t)}_s)  \rev{g_s^\star}\\
& \leq -f_{y_{s}}^*( \theta^{\star,(t)}_s) + f_{y_{s}}^*( \theta_{s}) - ( a_t \theta_{s} - \theta^{\star,(t)}_s) \rev{g_s^\star}  - (1-a_t)   \theta_{s}  \rev{g_s^\star}
\end{aligned}
\end{equation}
where $a_t = \min\pa{1,1/\Omega^D(\bcert)}$ is such that
$
a_t (\theta_s)_{s\leq t} \in \Kk_{\lambda,\eta^{(t)}}$, the dual constraint set defined in \eqref{eq:finitesumdual}.
Since $\theta^{\star,(t)}$ is a dual optimal point, by Fermat's rule
\[
-\msum_{s\leq t} \eta_s^{(t)} ( a_t\theta_{s} - \theta^{\star,(t)}_s)  \rev{g_s^\star} \leq 0  . 
\]
Therefore,
 multiplying \eqref{eq:b1} by $\eta_s^{(t)}$ and summing from $s=1,\ldots, t$, we obtain
\begin{equation}\label{eq:comp}
\begin{aligned}
\sfrac{1}{2L} \msum_{s=1}^t \eta^{(t)}_s \abs{ \theta_{s} - \theta^{\star,(t)}_s}^2  
\leq  &\msum_{s=1}^t \eta^{(t)}_s \Pa{-f_{y_{s}}^*( \theta^{\star,(t)}_s) + f_{y_{s}}^*( \theta_{s})} - (1-a_t) \msum_{s=1}^t \eta_s^{(t)}    \theta_{s}  \rev{g_s^\star}. 
\end{aligned}
\end{equation}
Note that by optimality of $\theta^{\star,(t)}$,  $ \sum_{s=1}^t \eta^{(t)}_s \bPa{-f_{y_{s}}^*( \theta^{\star,(t)}_s)} \leq  \bar P^{(t)}(\beta)$
for all $\beta\in\RR^d$. \rev{Therefore, the first sum in the RHS of \eqref{eq:comp} is bounded by $\gap_t(\beta)$.} 
\rev{To bound the second sum, we make the choice  $g_s^\star = x_s^\top \beta^{\star,(t)} $:
 Recall that since $f_{y_s}$ is proper, closed and convex, $v =   f_{y_s}'(x)$ if and only if $x\in \partial f_{y_s}^*(v)$ \citep[Prop 11.3]{rockafellar2009variational}. Hence for  $\beta^{\star,(t)} \in \Argmin_\beta \bar P^{(t)}(\beta)$, we have $ \theta^{\star,(t)}_s =  f_{y_s}'(x_s^\top \beta^{\star,(t)})$ and $x_s^\top \beta^{\star,(t)} \in \partial f_{y_s}^*(\theta^{\star,(t)}_s)$. We can therefore choose $g_s^\star = x_s^\top \beta^{\star,(t)} $. It then follows that
}
\[
\abs{\msum_{s=1}^t \eta_s^{(t)}   \rev{g_s^\star} \theta_s} 
=\abs{ \pa{  \ssum_{s=1}^t \eta_s^{(t)}  \theta_s x_s }^\top \beta^{\star,(t)} } \rev{\leq \Omega^D\Pa{\ssum_{s=1}^t \eta_s^{(t)}  \theta_s x_s}\cdot \Omega(\beta^{\star,(t)})}.
\]
By optimality of $\beta^{\star,(t)}$, we have \finalrev{
$
\Omega(\beta^{\star,(t)}) \leq \sfrac{1}{\lambda} \bar P^{(t)}(\hat \beta)$ for any $\hat \beta$}. 
Plugging these estimates back into \eqref{eq:comp} yields, for any $\beta\in \RR^n$,
\begin{equation}\label{eq:comp2}
\sfrac{1}{2L} \msum_{s=1}^t \eta^{(t)}_s \abs{ \theta_{s} - \theta^{\star,(t)}_s}^2  \leq \gap_t(\beta) + \finalrev{\bar P^{(t)}(\hat \beta)} \; ( \Omega^D(\bcert)-1)_+ 
\end{equation}
since $(1-a_t) \Omega^D(\bcert) =  ( \Omega^D(\bcert)-1)_+ $.
Finally, by the Cauchy-Schwarz inequality,
\[
\begin{aligned}
\Omega_g^D\bPa{\sfrac{1}{\lambda} \ssum_{s=1}^t \eta^{(t)}_s  ( \theta_s x_{s} -\theta^{\star,(t)}_s x_{s})}
&= \sfrac{1}{\lambda} \sup_{\Omega_g(z) \leq 1} \dotp{\ssum_{s=1}^t \eta^{(t)}_s  ( \theta_s x_{s} -\theta^{\star,(t)}_t x_{s})}{z}\\
&\leq \sfrac{1}{\lambda} \sqrt{\ssum_{s=1}^t \eta^{(t)}_s \abs{ \theta_s -\theta^{\star,(t)}_s}^2} \sup_{\Omega_g(z) \leq 1}   \sqrt{ \ssum_{s=1}^t \eta^{(t)}_s \abs{\dotp{ x_{s}}{z}}^2}\\
&\leq \sfrac{1}{\lambda} \sqrt{\ssum_{s=1}^t \eta^{(t)}_s \abs{ \theta_s -\theta^{\star,(t)}_s}^2}  \sqrt{ \ssum_{s=1}^t \eta^{(t)}_s \Omega^D_g( x_{s})^2}
\end{aligned}\]
and the result follows by combining this with \eqref{eq:comp2}. 
\end{proof}

\begin{small}
\setlength{\bibsep}{0pt plus 0.3ex}
\bibliographystyle{chicago}
\bibliography{bib}

\begin{thebibliography}{}

\bibitem[\protect\citeauthoryear{Andrews}{Andrews}{1992}]{andrews1992generic}
Andrews, D.~W. (1992).
\newblock Generic uniform convergence.
\newblock {\em Econometric theory\/}~{\em 8\/}(2), 241--257.

\bibitem[\protect\citeauthoryear{Bao, Gu, and Huang}{Bao
  et~al.}{2020}]{bao2020fast}
Bao, R., B.~Gu, and H.~Huang (2020).
\newblock Fast oscar and owl regression via safe screening rules.
\newblock In {\em International Conference on Machine Learning}, pp.\
  653--663. PMLR.

\bibitem[\protect\citeauthoryear{Bertsekas}{Bertsekas}{1999}]{bertsekas1999nonlinear}
Bertsekas, D.~P. (1999).
\newblock {\em Nonlinear programming}.
\newblock Athena scientific Belmont.

\bibitem[\protect\citeauthoryear{Bonnefoy, Emiya, Ralaivola, and
  Gribonval}{Bonnefoy et~al.}{2015}]{bonnefoy2015dynamic}
Bonnefoy, A., V.~Emiya, L.~Ralaivola, and R.~Gribonval (2015).
\newblock Dynamic screening: Accelerating first-order algorithms for the lasso
  and group-lasso.
\newblock {\em IEEE Transactions on Signal Processing\/}~{\em 63\/}(19),
  5121--5132.

\bibitem[\protect\citeauthoryear{Bottou and Cun}{Bottou and
  Cun}{2004}]{bottou2004large}
Bottou, L. and Y.~L. Cun (2004).
\newblock Large scale online learning.
\newblock In {\em Advances in neural information processing systems}, pp.\
  217--224.

\bibitem[\protect\citeauthoryear{Combettes and Pesquet}{Combettes and
  Pesquet}{2011}]{combettes2011proximal}
Combettes, P.~L. and J.-C. Pesquet (2011).
\newblock Proximal splitting methods in signal processing.
\newblock In {\em Fixed-point algorithms for inverse problems in science and
  engineering}, pp.\  185--212. Springer.

\bibitem[\protect\citeauthoryear{Defazio, Bach, and Lacoste-Julien}{Defazio
  et~al.}{2014}]{saga14}
Defazio, A., F.~Bach, and S.~Lacoste-Julien (2014).
\newblock Saga: A fast incremental gradient method with support for
  non-strongly convex composite objectives.
\newblock In {\em Advances in Neural Information Processing Systems}, pp.\
  1646--1654.

\bibitem[\protect\citeauthoryear{Duchi and Singer}{Duchi and
  Singer}{2009}]{duchi2009efficient}
Duchi, J. and Y.~Singer (2009).
\newblock Efficient online and batch learning using forward backward splitting.
\newblock {\em Journal of Machine Learning Research\/}~{\em 10\/}(Dec),
  2899--2934.

\bibitem[\protect\citeauthoryear{El~Ghaoui, Viallon, and Rabbani}{El~Ghaoui
  et~al.}{2010}]{ghaoui2010safe}
El~Ghaoui, L., V.~Viallon, and T.~Rabbani (2010).
\newblock Safe feature elimination for the lasso and sparse supervised learning
  problems.
\newblock {\em arXiv preprint arXiv:1009.4219\/}.

\bibitem[\protect\citeauthoryear{Hastie, Tibshirani, and Wainwright}{Hastie
  et~al.}{2015}]{hastie2015statistical}
Hastie, T., R.~Tibshirani, and M.~Wainwright (2015).
\newblock {\em Statistical learning with sparsity: the lasso and
  generalizations}.
\newblock CRC press.

\bibitem[\protect\citeauthoryear{Johnson and Zhang}{Johnson and
  Zhang}{2013}]{johnson2013accelerating}
Johnson, R. and T.~Zhang (2013).
\newblock Accelerating stochastic gradient descent using predictive variance
  reduction.
\newblock In {\em Advances in neural information processing systems}, pp.\
  315--323.

\bibitem[\protect\citeauthoryear{Kiefer, Wolfowitz, et~al.}{Kiefer
  et~al.}{1952}]{kiefer1952stochastic}
Kiefer, J., J.~Wolfowitz, et~al. (1952).
\newblock Stochastic estimation of the maximum of a regression function.
\newblock {\em The Annals of Mathematical Statistics\/}~{\em 23\/}(3),
  462--466.

\bibitem[\protect\citeauthoryear{Klopfenstein, Bertrand, Gramfort, Salmon, and
  Vaiter}{Klopfenstein et~al.}{2020}]{klopfenstein2020model}
Klopfenstein, Q., Q.~Bertrand, A.~Gramfort, J.~Salmon, and S.~Vaiter (2020).
\newblock Model identification and local linear convergence of coordinate
  descent.
\newblock {\em arXiv preprint arXiv:2010.11825\/}.

\bibitem[\protect\citeauthoryear{Langford, Li, and Zhang}{Langford
  et~al.}{2009}]{langford2009sparse}
Langford, J., L.~Li, and T.~Zhang (2009).
\newblock Sparse online learning via truncated gradient.
\newblock {\em Journal of Machine Learning Research\/}~{\em 10\/}(Mar),
  777--801.

\bibitem[\protect\citeauthoryear{Lee and Wright}{Lee and
  Wright}{2012}]{lee2012manifold}
Lee, S. and S.~J. Wright (2012).
\newblock Manifold identification in dual averaging for regularized stochastic
  online learning.
\newblock {\em Journal of Machine Learning Research\/}~{\em 13\/}(Jun),
  1705--1744.

\bibitem[\protect\citeauthoryear{Lewis}{Lewis}{2002}]{lewis2002active}
Lewis, A.~S. (2002).
\newblock Active sets, nonsmoothness, and sensitivity.
\newblock {\em SIAM Journal on Optimization\/}~{\em 13\/}(3), 702--725.

\bibitem[\protect\citeauthoryear{Lewis and Wright}{Lewis and
  Wright}{2016}]{lewis2016proximal}
Lewis, A.~S. and S.~J. Wright (2016).
\newblock A proximal method for composite minimization.
\newblock {\em Mathematical Programming\/}~{\em 158\/}(1-2), 501--546.

\bibitem[\protect\citeauthoryear{Liang, Fadili, and Peyr{\'e}}{Liang
  et~al.}{2017}]{liang2017activity}
Liang, J., J.~Fadili, and G.~Peyr{\'e} (2017).
\newblock Activity identification and local linear convergence of
  {F}orward--{B}ackward-type methods.
\newblock {\em SIAM Journal on Optimization\/}~{\em 27\/}(1), 408--437.

\bibitem[\protect\citeauthoryear{Liu, Zhao, Wang, and Ye}{Liu
  et~al.}{2014}]{liu2014safe}
Liu, J., Z.~Zhao, J.~Wang, and J.~Ye (2014).
\newblock Safe screening with variational inequalities and its application to
  lasso.
\newblock In {\em International Conference on Machine Learning}, pp.\
  289--297.

\bibitem[\protect\citeauthoryear{Mairal}{Mairal}{2013}]{mairal2013stochastic}
Mairal, J. (2013).
\newblock Stochastic majorization-minimization algorithms for large-scale
  optimization.
\newblock In {\em Advances in Neural Information Processing Systems}, pp.\
  2283--2291.

\bibitem[\protect\citeauthoryear{Natarajan}{Natarajan}{2017}]{natarajan2017classical}
Natarajan, P.~N. (2017).
\newblock {\em Classical summability theory}.
\newblock Springer.

\bibitem[\protect\citeauthoryear{Ndiaye, Fercoq, Gramfort, and Salmon}{Ndiaye
  et~al.}{2015}]{ndiaye2015gap}
Ndiaye, E., O.~Fercoq, A.~Gramfort, and J.~Salmon (2015).
\newblock Gap safe screening rules for sparse multi-task and multi-class
  models.
\newblock In {\em Advances in neural information processing systems}, pp.\
  811--819.

\bibitem[\protect\citeauthoryear{Ndiaye, Fercoq, Gramfort, and Salmon}{Ndiaye
  et~al.}{2016}]{ndiaye2016gap}
Ndiaye, E., O.~Fercoq, A.~Gramfort, and J.~Salmon (2016).
\newblock Gap safe screening rules for sparse-group lasso.
\newblock In {\em Advances in Neural Information Processing Systems}, pp.\
  388--396.

\bibitem[\protect\citeauthoryear{Ndiaye, Fercoq, Gramfort, and Salmon}{Ndiaye
  et~al.}{2017}]{ndiaye2017gap}
Ndiaye, E., O.~Fercoq, A.~Gramfort, and J.~Salmon (2017).
\newblock Gap safe screening rules for sparsity enforcing penalties.
\newblock {\em The Journal of Machine Learning Research\/}~{\em 18\/}(1),
  4671--4703.

\bibitem[\protect\citeauthoryear{Poon, Liang, and Schoenlieb}{Poon
  et~al.}{2018}]{poon2018local}
Poon, C., J.~Liang, and C.-B. Schoenlieb (2018).
\newblock Local convergence properties of saga/prox-svrg and acceleration.
\newblock In {\em Proceedings of the 35th International Conference on Machine
  Learning}, pp.\  4124--4132. PMLR.

\bibitem[\protect\citeauthoryear{Rao}{Rao}{2008}]{rao2008course}
Rao, S.~S. (2008).
\newblock A course in time series analysis.
\newblock {\em Technical Report, Texas A\&M University\/}.

\bibitem[\protect\citeauthoryear{Robbins and Monro}{Robbins and
  Monro}{1951}]{robbins1951stochastic}
Robbins, H. and S.~Monro (1951).
\newblock A stochastic approximation method.
\newblock {\em The annals of mathematical statistics\/}, 400--407.

\bibitem[\protect\citeauthoryear{Robbins and Siegmund}{Robbins and
  Siegmund}{1971}]{robbins1971convergence}
Robbins, H. and D.~Siegmund (1971).
\newblock A convergence theorem for non negative almost supermartingales and
  some applications.
\newblock In {\em Optimizing methods in statistics}, pp.\  233--257. Elsevier.

\bibitem[\protect\citeauthoryear{Rockafellar and Wets}{Rockafellar and
  Wets}{2009}]{rockafellar2009variational}
Rockafellar, R.~T. and R.~J.-B. Wets (2009).
\newblock {\em Variational analysis}, Volume 317.
\newblock Springer Science \& Business Media.

\bibitem[\protect\citeauthoryear{Simon, Friedman, Hastie, and Tibshirani}{Simon
  et~al.}{2013}]{simon2013sparse}
Simon, N., J.~Friedman, T.~Hastie, and R.~Tibshirani (2013).
\newblock A sparse-group lasso.
\newblock {\em Journal of computational and graphical statistics\/}~{\em
  22\/}(2), 231--245.

\bibitem[\protect\citeauthoryear{Sun and Bach}{Sun and
  Bach}{2020}]{sun2020safe}
Sun, Y. and F.~Bach (2020).
\newblock Safe screening for the generalized conditional gradient method.
\newblock {\em arXiv preprint arXiv:2002.09718\/}.

\bibitem[\protect\citeauthoryear{Tibshirani}{Tibshirani}{1996}]{tibshirani1996regression}
Tibshirani, R. (1996).
\newblock Regression shrinkage and selection via the lasso.
\newblock {\em Journal of the Royal Statistical Society: Series B
  (Methodological)\/}~{\em 58\/}(1), 267--288.

\bibitem[\protect\citeauthoryear{Tibshirani, Bien, Friedman, Hastie, Simon,
  Taylor, and Tibshirani}{Tibshirani et~al.}{2012}]{tibshirani2012strong}
Tibshirani, R., J.~Bien, J.~Friedman, T.~Hastie, N.~Simon, J.~Taylor, and R.~J.
  Tibshirani (2012).
\newblock Strong rules for discarding predictors in lasso-type problems.
\newblock {\em Journal of the Royal Statistical Society: Series B (Statistical
  Methodology)\/}~{\em 74\/}(2), 245--266.

\bibitem[\protect\citeauthoryear{Vaiter, Golbabaee, Fadili, and
  Peyr{\'e}}{Vaiter et~al.}{2015}]{vaiter2015model}
Vaiter, S., M.~Golbabaee, J.~Fadili, and G.~Peyr{\'e} (2015).
\newblock Model selection with low complexity priors.
\newblock {\em Information and Inference: A Journal of the IMA\/}~{\em 4\/}(3),
  230--287.

\bibitem[\protect\citeauthoryear{Wang, Zhou, Liu, Wonka, and Ye}{Wang
  et~al.}{2014}]{wang2014safe}
Wang, J., J.~Zhou, J.~Liu, P.~Wonka, and J.~Ye (2014).
\newblock A safe screening rule for sparse logistic regression.
\newblock In {\em Advances in neural information processing systems}, pp.\
  1053--1061.

\bibitem[\protect\citeauthoryear{Wang, Zhou, Wonka, and Ye}{Wang
  et~al.}{2013}]{wang2013lasso}
Wang, J., J.~Zhou, P.~Wonka, and J.~Ye (2013).
\newblock Lasso screening rules via dual polytope projection.
\newblock In {\em Advances in neural information processing systems}, pp.\
  1070--1078.

\bibitem[\protect\citeauthoryear{Xiang, Wang, and Ramadge}{Xiang
  et~al.}{2016}]{xiang2016screening}
Xiang, Z.~J., Y.~Wang, and P.~J. Ramadge (2016).
\newblock Screening tests for lasso problems.
\newblock {\em IEEE transactions on pattern analysis and machine
  intelligence\/}~{\em 39\/}(5), 1008--1027.

\bibitem[\protect\citeauthoryear{Xiang, Xu, and Ramadge}{Xiang
  et~al.}{2011}]{xiang2011learning}
Xiang, Z.~J., H.~Xu, and P.~J. Ramadge (2011).
\newblock Learning sparse representations of high dimensional data on large
  scale dictionaries.
\newblock In {\em Advances in neural information processing systems}, pp.\
  900--908.

\bibitem[\protect\citeauthoryear{Xiao}{Xiao}{2010}]{xiao2010dual}
Xiao, L. (2010).
\newblock Dual averaging methods for regularized stochastic learning and online
  optimization.
\newblock {\em Journal of Machine Learning Research\/}~{\em 11\/}(Oct),
  2543--2596.

\bibitem[\protect\citeauthoryear{Yuan and Lin}{Yuan and
  Lin}{2006}]{yuan2006model}
Yuan, M. and Y.~Lin (2006).
\newblock Model selection and estimation in regression with grouped variables.
\newblock {\em Journal of the Royal Statistical Society: Series B (Statistical
  Methodology)\/}~{\em 68\/}(1), 49--67.

\end{thebibliography}
\end{small}

\end{document}